\def\MODE{1} 
\if\MODE1
\newcommand{\COND}[2]{#1}
\else
\newcommand{\COND}[2]{#2}
\fi

\documentclass{article}

\usepackage{times}

\usepackage{graphicx} 
\usepackage{subfigure}

\usepackage{algorithm}
\usepackage{algorithmic}

\usepackage[
  pdfauthor={Stephen Tu, Rebecca Roelofs, Shivaram Venkataraman, and Benjamin Recht},
  pdftitle={Large Scale Kernel Learning using Block Coordinate Descent},
  colorlinks=true,citecolor=blue,linkcolor=blue]{hyperref}

\usepackage{upgreek}
\newcommand{\kernel}{\upkappa}

\usepackage{fullpage}

\usepackage{graphicx, subfigure, multirow, times, units, hhline, amsthm}
\usepackage{url,amssymb,amsmath,amscd, amsfonts, verbatim, mathtools, color, xspace, caption, graphicx,fancyhdr, xypic, xy, enumerate, mathrsfs,bm}

\newcommand*\samethanks[1][\value{footnote}]{\footnotemark[#1]}

\newtheorem{theorem}{Theorem}[section]
\newtheorem{lemma}[theorem]{Lemma}

\newtheorem{proposition}[theorem]{Proposition}

\newcommand{\vct}[1]{#1}
\newcommand{\mtx}[1]{#1}

\DeclareMathOperator*{\argmin}{arg\!\min}
\DeclareMathOperator*{\rank}{rank}
\DeclareMathOperator*{\Span}{span}
\DeclareMathOperator*{\diag}{diag}
\newcommand{\Tr}{\mathrm{Tr}}

\newcommand{\beq}{\begin{equation}}
\newcommand{\eeq}{\end{equation}}
\newcommand{\beqs}{\begin{equation*}}
\newcommand{\eeqs}{\end{equation*}}

\newcommand{\nystrom}{Nystr\"{o}m}

\newcommand{\R}{\mathbb{R}}
\newcommand{\E}{\mathbb{E}}
\newcommand{\X}{\mathcal{X}}
\newcommand{\HS}{\mathcal{H}}
\newcommand{\ind}{\mathbf{1}}
\newcommand{\ip}[2]{\ensuremath{\langle #1, #2 \rangle}}

\newcommand{\norm}[1]{\lVert #1 \rVert}

\newcommand{\abs}[1]{\ensuremath{| #1 |}}

\newcommand{\T}{\mathsf{T}}
\renewcommand{\Pr}{\mathbb{P}}

\newcommand{\citesup}{\cite}

\title{Large Scale Kernel Learning using Block Coordinate Descent}
\author{
  Stephen Tu~\thanks{Department of Electrical Engineering and Computer Science, UC Berkeley, Berkeley, CA.} \quad
  Rebecca Roelofs~\samethanks[1] \quad
  Shivaram Venkataraman~\samethanks[1] \quad
  Benjamin Recht~\samethanks[1]~\thanks{Department of Statistics, UC Berkeley, Berkeley, CA.}
}
\date{\today}

\def\abovestrut#1{\rule[0in]{0in}{#1}\ignorespaces}
\def\belowstrut#1{\rule[-#1]{0in}{#1}\ignorespaces}
\def\abovespace{\abovestrut{0.20in}}

\def\belowspace{\belowstrut{0.10in}}

\begin{document}

\maketitle
\thispagestyle{empty}

\begin{abstract}
We demonstrate that distributed block coordinate descent can quickly solve
kernel regression and classification problems with millions of data points.
Armed with this capability, we conduct a thorough comparison between the full kernel, the \nystrom{}
method, and random features on three large classification tasks from various domains. 
Our results suggest that the \nystrom{} method
generally achieves better statistical accuracy than random features, but can require significantly
more iterations of optimization.
Lastly, we derive new rates for block coordinate descent which support our experimental findings
when specialized to kernel methods.
\end{abstract}

\section{Introduction}
\label{sec:intro}

Kernel methods are a powerful tool in machine learning, allowing one to
discover non-linear structure by mapping data into a higher dimensional,
possibly infinite, feature space.  
However, a known issue is that kernel methods do not scale favorably with
dataset size.  For instance, a na{\"{i}}ve implementation of a kernel least
squares solver requires $O(n^2)$ space and $O(n^3)$ time to store and invert
the full kernel matrix. The prevailing belief is that when $n$ reaches the
millions, kernel methods are impractical. 

This paper challenges the conventional wisdom by pushing kernel methods
to the limit of what is practical on modern distributed compute platforms.
We show that approximately solving a full kernel least squares problem with
$n=2 \times 10^6$ can be done in a matter of hours, and the resulting model
achieves competitive performance in terms of classification errors.
Mimicking the successes of the early 2000s, our algorithm is based on block coordinate descent and
avoids full materialization of the kernel matrix~\cite{joachims1999svmlight,fan05}. 

Furthermore, in contrast to running multiple iterations in parallel and aggregating updates~\cite{agarwal11,
niu11, zinkevich11, jaggi14, liu2015asynchronous}, we exploit distributed
computation to parallelize \emph{individual iterations} of block coordinate descent.
We deliberately make this choice to alleviate communication overheads.
Our resulting implementation inherits the linear convergence of 
block coordinate descent while efficiently scaling up to 1024 cores on 128 machines.

The capability to solve full kernel systems allows us to perform a
direct head-to-head empirical comparison between popular kernel approximation
techniques and the full kernel at an unprecedented scale. We conduct a thorough study of random features
\cite{rahimi07} and \nystrom{} \cite{williams2001using} approximations on three
large datasets from speech, text, and image classification domains.  Extending
prior work comparing kernel approximations~\cite{yang12},  our study is the first to
work with multi-terabyte kernel matrices and to quantify
computational versus statistical performance tradeoffs between the two methods
at this scale. More specifically, we identify situations where the \nystrom{}
system requires significantly more iterations to converge than a random
features system of the same size, but yields a better estimator when it does.

Finally, motivated by the empirical effectiveness of primal block coordinate
descent methods in our own study and in related work that inspired our
investigations~\cite{huang2014kernel}, we derive a new rate of convergence for
block coordinate descent on strongly convex smooth quadratic functions. Our
analysis shows that block coordinate descent has a convergence rate that is no
worse than gradient descent plus a small additive factor which is inversely
proportional to the block size.  Specializing this result to random features,
\nystrom{}, and kernel risk minimization problems corroborates our experimental
findings regarding the iteration complexity of the three methods.

\section{Background}
This section concisely overviews the techniques used in this paper, and more importantly defines the
specific optimization problems we solve.
The theoretical underpinnings of kernel methods and their various approximations are
well established in the literature; see e.g. \cite{scholkopf01} for a thorough
treatment.

\paragraph{Notation.} For a vector $\vct{x}$, we let $\norm{\vct{x}}$ denote
the Euclidean norm.  For a matrix $\mtx{X}$, we let $\norm{\mtx{X}}$
denote the operator norm, $\norm{\mtx{X}}_F$ the
Frobenius norm, and $\sigma_1(X) \geq \sigma_2(X) \geq ... \geq
\sigma_{r}(X) > 0$ denote the singular values of $X$ in decreasing order,
where $r = \rank(X)$.  If $X$ is symmetric, let $\lambda_{\max}(X),
\lambda_{\min}(X)$ denote the maximum and minimum eigenvalues of $X$,
respectively.  For two conforming matrices $\mtx{A}$ and $\mtx{B}$,
$\ip{\mtx{A}}{\mtx{B}} := \Tr(\mtx{A}^\T \mtx{B})$.

Finally, given a matrix $X \in \R^{n_1 \times n_2}$ and two index sets $I \in 2^{[n_1]},
J \in 2^{[n_2]}$, we let $X(I, J) \in \R^{\abs{I} \times \abs{J}}$ denote the
submatrix of $X$ which selects out the rows in $I$ and the columns in $J$.

\subsection{Kernel methods and approximations}
\label{sec:background:kernelmethods}
 Let
$\HS$ be a reproducing kernel Hilbert space (RKHS) of functions $f : \X
\rightarrow \R$, with associated Mercer kernel $\kernel : \X \times \X \rightarrow
\R$. We typically associate $\X$ with $\R^d$.

Given a set of data points $\{(\vct{x_i}, y_i)\}_{i=1}^{n}$ with $\vct{x_i} \in \X$
and $y_i \in \{1, 2, ..., k\}$, we use the standard one-versus-all (OVA)
approach \cite{rifkin04} to turn a multiclass classification problem into
$k$ binary classification problems of the form
\beq
  \min_{f_j \in \HS} \frac{1}{n} \sum_{i=1}^{n} \ell( f_j(\vct{x_i}), y_{ij} ) + \lambda \norm{f_j}_\HS^2, \;\; j=1, ..., k \:, \label{eq:classification}
\eeq
where $y_{ij}$ is $1$ if $y_i=j$ and $-1$ otherwise. While in general
$\ell$ can be any convex loss function, we focus on the square loss
$\ell(a, b) = (a-b)^2$ to make the algorithmic and systems comparisons more transparent.
While other loss functions like softmax, logistic, or hinge losses are frequently used for 
classification, regularized least squares classification performs as well in most scenarios ~\cite{rifkin2002everything,
rahimi07, agarwal2013least}.  Furthermore, 
a least squares solver can be bootstrapped into a minimizer for general loss functions 
with little additional cost
using a splitting method such as ADMM \cite{boyd2011distributed, zhang2014materialization}.

Owing to the representer theorem, minimization over $\HS$ in
\eqref{eq:classification} is equivalent to minimization over $\HS_n$, where
$\HS_n := \Span\{ \kernel(\vct{x_i}, \cdot) : i=1,...,n \}$.  Therefore, defining
$\mtx{K} \in \R^{n \times n}$ as $\mtx{K}_{ij} := \kernel(\vct{x_i}, \vct{x_j})$, we
can write \eqref{eq:classification} as
\beq
  \min_{ \mtx{\alpha} \in \R^{n \times k} } \frac{1}{n} \norm{ \mtx{K}\mtx{\alpha} - \mtx{Y} }_F^2 + \lambda \ip{\mtx{\alpha}}{\mtx{K}\mtx{\alpha}} \:, \label{eq:kernells}
\eeq
where $\mtx{Y} \in \R^{n \times k}$ is a label matrix
with $\mtx{Y}_{ij} = 1$ if $y_i = j$ and $-1$ otherwise. 
The normal equation of \eqref{eq:kernells} is
\beqs
  K(K+n\lambda I_n)\alpha = KY \:. 
\eeqs
Solutions of
\eqref{eq:kernells} take on the form
$\mtx{\alpha}_* = (\mtx{K} + n\lambda \mtx{I}_n)^{-1} \mtx{Y} + \mtx{Q}$,
with $Q \in \R^{n\times k}$ satisfying $\mtx{K}\mtx{Q} = 0_{n \times k}$.
The resulting $f$ is $f(\vct{x}) = (\kernel(\vct{x}, \vct{x_1}), ..., \kernel(\vct{x}, \vct{x_n}))^\T \mtx{\alpha_*} \in \R^{1 \times k}$.

\paragraph{\nystrom{} method.}
Let $I \in 2^{[n]}$ denote an index set of size $p$, and let $\mtx{K}_I := K([n],I)$, $\mtx{K}_{II} := K(I, I)$.
One common variant of the \nystrom{} method
\cite{williams2001using,drineas05,gittens13,bach05} is to use
the matrix $\widehat{K}:= K_I K_{II}^\dag K_I^\T$ as a low rank approximation to $K$
\eqref{eq:kernells}. An alternative approach is to first 
replace the minimization over $\HS$ in \eqref{eq:classification} with $\HS_I$,
where $\HS_I := \Span\{ \kernel(\vct{x_i}, \cdot) : i \in I \}$.
We then arrive at the optimization problem
\beq
  \min_{ \mtx{\alpha} \in \R^{p \times k} } \frac{1}{n} \norm{ \mtx{K}_I\mtx{\alpha} - \mtx{Y} }_F^2 + \lambda \ip{\mtx{\alpha}}{\mtx{K}_{II}\mtx{\alpha}} \:, \label{eq:kernelnystrom}
\eeq
with $f(\vct{x}) = (\kernel(\vct{x}, \vct{x_{I(1)}}), ..., \kernel(\vct{x}, \vct{x_{I(p)}}))^\T \mtx{\alpha_*} \in \R^{1 \times k}$. The normal equation for \eqref{eq:kernelnystrom} is
\beqs
  (\mtx{K}_I^\T \mtx{K}_I + n\lambda \mtx{K}_{II}) \mtx{\alpha} = \mtx{K}_I^\T \mtx{Y} \:,
\eeqs
and hence solutions take on the form
$\mtx{\alpha_*} = (\mtx{K}_I^\T \mtx{K}_I + n\lambda \mtx{K}_{II})^\dag \mtx{K}_I^\T \mtx{Y} + \mtx{Q}$,
with $Q \in \R^{p \times k}$ satisfying $\mtx{K}_I \mtx{Q} = 0_{p \times k}$.
For numerical stability reasons, one might pick a small $\gamma > 0$ and solve instead
\beqs
  (\mtx{K}_I^\T \mtx{K}_I + n\lambda \mtx{K}_{II} + n\lambda \gamma \mtx{I}_{p} )\mtx{\alpha} = \mtx{K}_I^\T \mtx{Y} \:.
\eeqs
This extra regularization is justified statistically
by \cite{bach13,alaoui15}.

\paragraph{Random features.}
Random feature based methods \cite{rahimi07} use an element-wise approximation of $\mtx{K}$.
Suppose that $(\Omega, \rho)$ is a measure space and 
$\varphi : \X \times \Omega \rightarrow \R$ is a measurable function such that for all
$\vct{x}, \vct{y} \in \X$,
$\E_{\omega \sim \rho} \varphi(\vct{x}, \omega) \varphi(\vct{y}, \omega) = \kernel(\vct{x}, \vct{y})$.
Random feature approximations works by drawing $\omega_1, ..., \omega_p
\stackrel{\mathrm{iid}}{\sim} \rho$ and defining the map $z : \X \rightarrow
\R^p$ as
\beqs
  z(\vct{x}) := \frac{1}{\sqrt{p}} (\varphi(\vct{x}, \omega_1), ..., \varphi(\vct{x}, \omega_p)) \:.
\eeqs
The optimization of $f$ in \eqref{eq:classification} is then restricted to the
space $\HS_{\rho} := \Span\{ z(\vct{x_i})^\T z(\cdot) : i=1,...,n \}$.  Define
$\mtx{Z} \in \R^{n \times p}$ as $\mtx{Z} := (z(\vct{x_1}), ...,
z(\vct{x_n}))^\T$. 
Applying the same argument as before followed by an appropriate
change of variables,
we can solve the program in primal form
\beq
  \min_{\mtx{w} \in \R^{p \times k}} \frac{1}{n} \norm{\mtx{Z} \mtx{w} - \mtx{Y}}_F^2 + \lambda \norm{\mtx{w}}_F^2 \:. \label{eq:randomfeatures}
\eeq
The normal equation for \eqref{eq:randomfeatures} is
\beqs
  (Z^\T Z + n\lambda I_p) w = Z^\T Y \:,
\eeqs
and hence
$\mtx{w}_* = (\mtx{Z}^\T \mtx{Z} + n\lambda \mtx{I}_p)^{-1} \mtx{Z}^\T \mtx{Y}$
and $f(\vct{x}) = z(\vct{x})^\T \mtx{w}_* \in \R^{1 \times k}$.

Note that when $\X = \R^d$ and $\kernel(\vct{x}, \vct{y}) = \kernel(\norm{\vct{x} -
\vct{y}})$ is translation invariant, Bochner's theorem states that the
(scaled) Fourier transform of $\kernel(\cdot)$ will be a valid probability measure on $\R^d$.
The map $\varphi$ can then be constructed as
$\varphi(x, (\omega, b)) = \sqrt{2}\cos( x^\T \omega + b)$, where
$\omega$ is drawn from the Fourier transform of $\kernel(\cdot)$ and $b \sim \mathrm{Unif}([0, 2\pi])$.

\subsection{Related work}

An empirical comparison on \nystrom{} versus random features was done by Yang et
al.~\cite{yang12}. This study demonstrated that the \nystrom{} method outperformed
random features on every dataset in their experiments. Our experimental efforts differ from
this seminal work in several ways. First, we quantify time versus statistical
performance tradeoffs, instead of studying only the empirical risk minimizer.
Second, we describe a scalable algorithm which allows us to compare performance
with the full kernel. Finally, our datasets are significantly larger, and we
also sweep across a much wider range of number of random features.  

On the algorithms side, the inspiration for this work was by Huang et
al.~\cite{huang2014kernel}, who devised a similar block coordinate descent
algorithm for solving random feature systems. In this work, we extend the block
coordinate algorithm to both the full kernel and \nystrom{} systems.
This enables us to train the full kernel on the entire TIMIT dataset, 
achieving a lower test error than the random feature approximations.

\section{Algorithms}
\label{sec:algorithms}

The optimal solutions written in Section~\ref{sec:background:kernelmethods}
require solving large linear systems where the data cannot be assumed to fit
entirely in memory. This necessitates a different algorithm than the least
squares solvers implemented in standard library routines.  Fortunately, for the
statistical problems we are interested in, obtaining a high accuracy solution
is not as important.  Hence, we propose to use block coordinate descent
\cite{bertsekas1989parallel}, which admits a natural distributed
implementation, and, in our experience, converges to a reasonable accuracy
after only a few passes through the data.

\paragraph{Coordinate methods in machine learning.}
Coordinate methods in machine learning date back to
the late 90s with SVMLight \cite{joachims1999svmlight} and SMO \cite{platt98}.
More recently, many researchers \cite{yang13,richtarik13,jaggi14,ma15} have proposed using
distributed computation to run multiple iterations of coordinate descent in parallel.
As noted previously, we take a different approach and use distributed computing
to accelerate within an iteration.  This is similar to \cite{hsieh08,yu10},
both who describe block coordinate algorithms for solving SVMs.  However, using
the square loss instead of hinge loss simplifies our analysis and
implementation.

\subsection{Block coordinate descent}
We first describe block coordinate descent generically and then specialize it for the
least squares loss. Let $f : \R^d \rightarrow \R$ be a twice
differentiable strongly convex, smooth function, and let $b \in \{1, ..., d\}$
denote a block size.  Let $I \in 2^{[d]}$ be an index set such that
$\abs{I} = b$, and let $P_I : \R^{d} \rightarrow \R^{d}$ be the projection
operator which zeros out all coordinates $j \not\in I$, leaving coordinates $i \in I$
intact. Block coordinate
descent works by iterating the mapping
\beqs
  \vct{w}^{\tau+1} \gets \vct{w}^{\tau} - \mtx{\Gamma}_\tau \cdot P_{I_\tau} \nabla f(\vct{w}^\tau) \:,
\eeqs
where $I_\tau$ is drawn at random by some sampling strategy (typically uniform),
and $\mtx{\Gamma}_\tau \in \R^{d \times d}$ is either fixed, or chosen by direct line search.
We choose the latter, in which case we write
\beq
  \vct{w}^{\tau+1} \gets \argmin_{\vct{w} \in \R^{d}} f( P_{I_\tau^c} \vct{w}^\tau + P_{I_\tau} \vct{w} ) \:. \label{eq:bcd_direct_solve_update}
\eeq
In the case where $f$ is least squares, the update
\eqref{eq:bcd_direct_solve_update} is equivalent to block Gauss-Seidel
on the normal equations. For instance, for \eqref{eq:randomfeatures}, the update
\eqref{eq:bcd_direct_solve_update} reduces to solving the $b \times b$ equation
\beq
  \vct{w}^{\tau+1}_{I_\tau} \gets (\mtx{Z}_{I_\tau}^\T \mtx{Z}_{I_\tau} + n\lambda \mtx{I}_b)^{-1} \mtx{Z}_{I_\tau}^\T \mtx{Y} \:, \label{eq:bcd_direct_solve_lapack}
\eeq
where $Z_{I_\tau} := Z([n], I_\tau)$.
The $w_{I_\tau}$ notation means we set only the coordinates in $I_\tau$ equal
to the RHS, and the coordinates not in $I_\tau$ remain the same from the previous
iteration.

\paragraph{Distributed execution.}
We solve block coordinate descent in parallel by distributing the computation of
$\mtx{Z}_{I_\tau}^\T \mtx{Z}_{I_\tau}$ and $\mtx{Z}_{I_\tau}^\T \mtx{Y}$. To do this, we partition
the rows of $\mtx{Z}_{I_\tau}$, $\mtx{Y}$ across all the machines in a cluster and compute the sum of
outer products from each machine. The result of this distributed operation is a $b \times b$ matrix
and we pick $b$ such that the solve for $w^{\tau+1}_{I_\tau}$ can be computed quickly using existing
\texttt{lapack} solvers on a single machine.

Choosing an appropriate value of $b$ is important as it affects both the statistical accuracy
and run-time performance. Using a larger value for $b$ leads to improved convergence and is also
helpful for using BLAS-3 primitives in single machine operations. However, a very large value for $b$
increases the serial execution time and the communication costs. In practice, we see that setting
$b$ in the range 2,000 to 8,000 offers a good trade-off.

\paragraph{Block generation primitives.}
As mentioned previously, our algorithms only require a procedure that
materializes a column block at a time.  We denote this primitive by
$\textsc{KernelBlock}(X, I)$, where $X$ represents the data matrix and $I$ is a
list of column indices. The output of $\textsc{KernelBlock}$ is $K([n], I)$.
After a column block is used in a block coordinate descent update of the model,
it can be immediately discarded.  We also use distributed computation to
parallelize the generation of a block $\mtx{K}_{I_\tau}$ of the kernel matrix.
We also define a similar primitive,
$\textsc{RandomFeaturesBlock}(X, I)$, which
returns $Z([n], I)$ for random feature systems.

\subsection{Algorithm descriptions}

\paragraph{Full kernel block coordinate descent.}
Our full kernel solver is described in Algorithm~\ref{alg:kernel_bcd}.
Algorithm~\ref{alg:kernel_bcd} is actually Gauss-Seidel on the linear system
$(K+n\lambda I_n) \alpha = Y$, but as we will discuss in
Section~\ref{sec:rates:primaldual}, this is equivalent to block coordinate
descent on a modified objective function (which is strongly convex, even when
$K$ is rank deficient). See \cite{hefny15} for a similar discussion in the context
of ridge regression.

\paragraph{\nystrom{} block coordinate descent.}
Unlike the full kernel case, our \nystrom{} implementation operates directly on the normal
equations.  A notable point of our algorithm is that it does \emph{not} require
computation of the pseudo-inverse $K_{II}^\dag$. When the number of \nystrom{}
features is large, calculating the pseudo-inverse $K_{II}^\dag$ is
expensive in terms of computation and communication. By making $K_{II}$ a part
of the block coordinate descent update we are able to handle large number of
\nystrom{} features while only needing a block of features at a time.

\newcommand{\TabletextCommComp}{
	Computation and communication costs for one epoch of distributed
	block coordinate descent. The number of examples is $n$,
	the number of features is $p$, the block
	size is $b$, the number of classes is $k$, and the number of machines is $M$.
	Each cost is presented as (cost for one block)$\times$(number of blocks).
}

\begin{table}[!t]
\COND{}{
\caption{\TabletextCommComp{}}
\label{tab:comm_comp_costs}
}
\begin{center}
\begin{small}
\begin{sc}
\setlength{\tabcolsep}{0.3em}
\begin{tabular}{lcc}
\hline
\abovespace\belowspace
Algorithm & Computation & Communication \\
\hline
\abovespace
Full kernel & $(\frac{nbk}{M} + b^3) \times \frac{n}{b}$ & $b^2 \times \frac{n}{b}$ \\
\nystrom{}/RF. & $(\frac{nb^2}{M} + \frac{nbk}{M} + b^3) \times \frac{p}{b}$ & $\log(M)b^2 \times \frac{p}{b}$  \\[1ex]
\hline
\end{tabular}
\end{sc}
\COND{
\caption{\TabletextCommComp{}}
\label{tab:comm_comp_costs}
}{}
\end{small}
\end{center}
\end{table}

We denote $\textsc{Selector}(n, I)$ as the function which
returns an $\{0, 1\}^{n \times \abs{I}}$ matrix $S$ such that $S_{I(j)j} = 1$
and zero otherwise, for $j=1, ..., \abs{I}$; this is simply the column selector
matrix associated with the indices in $I$. Using the above notation, the
\nystrom{} algorithm is described in Algorithm~\ref{alg:nystrom_bcd}.

\begin{algorithm}[tb]
   \caption{Full kernel block coordinate descent}
   \label{alg:kernel_bcd}
\begin{algorithmic}
   \STATE {\bfseries Input:} data $X \in \X^{n}$, $Y \in \{\pm 1\}^{n \times k}$,
   \STATE ~~~~~~~~~~~ number of epochs $n_e$,
   \STATE ~~~~~~~~~~~ block size $b \in \{1, ..., n\}$,
   \STATE ~~~~~~~~~~~ regularizer $\lambda > 0$.
   \STATE {\bfseries Assume:} $n/b$ is an integer.
   \STATE $\pi \gets$ random permutation of $\{1, ..., n\}$.
   \STATE $\mathcal{I}_1, ..., \mathcal{I}_{\frac{n}{b}} \gets$ partition $\pi$ into $\frac{n}{b}$ pieces.
   \STATE $\alpha \gets 0_{n \times k}$.
   \FOR{$\ell=1$ {\bfseries to} $n_e$}
       \STATE $\pi \gets$ random permutation of $\{1, ..., \frac{n}{b} \}$.
       \FOR{$i=1$ {\bfseries to} $\frac{n}{b}$}
           \STATE $K_b \gets \textsc{KernelBlock}(X, \mathcal{I}_{\pi_i})$.
           \STATE $Y_b \gets Y(\mathcal{I}_{\pi_i}, [k])$.
           \STATE $R \gets 0_{b \times k}$.
           \FOR{$j \in \{1, ..., \frac{n}{b} \} \setminus \{\pi_i\}$}
               \STATE $R \gets R + K_b(\mathcal{I}_{\pi_i}, [b])^\T \alpha(\mathcal{I}_{\pi_i}, [b])$.
           \ENDFOR
           \STATE Solve $(K_b(\mathcal{I}_{\pi_i}, [b]) + \lambda I_b)\alpha_b = Y_b - R$.
           \STATE $\alpha(\mathcal{I}_{\pi_i}, [k]) \gets \alpha_b$.
       \ENDFOR
   \ENDFOR
\end{algorithmic}
\end{algorithm}

\begin{algorithm}[tb]
   \caption{\nystrom{} block coordinate descent}
   \label{alg:nystrom_bcd}
\begin{algorithmic}
   \STATE {\bfseries Input:} data $X \in \X^{n}$, $Y \in \{\pm 1\}^{n \times k}$,
   \STATE ~~~~~~~~~~~ number of epochs $n_e$,
   \STATE ~~~~~~~~~~~ number of \nystrom{} features $p \in \{1, ..., n\}$,
   \STATE ~~~~~~~~~~~ block size $b \in \{1, ..., p\}$.
   \STATE ~~~~~~~~~~~ regularizers $\lambda > 0, \gamma \geq 0$.
   \STATE {\bfseries Assume:} $p/b$ is an integer.
   \STATE $\mathcal{J} \gets$ $p$ without replacement draws from $\{1, ..., n\}$.
   \STATE $\mathcal{I}_1, ..., \mathcal{I}_{\frac{p}{b}} \gets$ partition $\mathcal{J}$ into $\frac{p}{b}$ pieces.
   \STATE $\alpha \gets 0_{p \times k}$, $R \gets 0_{n \times k}$.
   \FOR{$\ell=1$ {\bfseries to} $n_e$}
       \STATE $\pi \gets$ random permutation of $\{1, ..., \frac{n}{b} \}$.
       \FOR{$i=1$ {\bfseries to} $\frac{p}{b}$}
           \STATE $B \gets \{ (\pi_i - 1)b + 1, ..., \pi_i b \}$.
           \STATE $\alpha_b \gets \alpha(B, [k])$.
           \STATE $S_b \gets \textsc{Selector}(n, \mathcal{I}_{\pi_i})$.
           \STATE $K_b \gets \textsc{KernelBlock}(X, \mathcal{I}_{\pi_i})$.
           \STATE $R \gets R - (K_b + n\lambda S_b) \alpha_b$.
           \STATE $K_{bb} \gets K_b(\mathcal{I}_{\pi_i}, [b])$.
           \STATE Solve $(K_b^\T K_b + n\lambda K_{bb} + n\lambda\gamma I_b) \alpha_b' = K_b^\T(Y - R)$.
           \STATE $R \gets R + (K_b + n\lambda S_b) \alpha_b'$.
           \STATE $\alpha(B, [k]) \gets \alpha_b'$.
       \ENDFOR
   \ENDFOR
\end{algorithmic}
\end{algorithm}

\paragraph{Random features block coordinate descent.}
\COND{
Our random features solver is the same as Algorithm~2 from
\cite{huang2014kernel}. We include it in
Algorithm~\ref{alg:rf_bcd} for completeness.
\begin{algorithm}[htb]
   \caption{Random features block coordinate descent}
   \label{alg:rf_bcd}
\begin{algorithmic}
   \STATE {\bfseries Input:} data $X \in \X^{n}$, $Y \in \{\pm 1\}^{n \times k}$.
   \STATE ~~~~~~~~~~~ number of epochs $n_e$,
   \STATE ~~~~~~~~~~~ number of random features $p \geq 1$,
   \STATE ~~~~~~~~~~~ block size $b \in \{1, ..., p\}$.
   \STATE ~~~~~~~~~~~ regularizers $\lambda > 0$.
   \STATE {\bfseries Assume:} $p/b$ is an integer.
   \STATE $\pi \gets $ random permutation of $\{1, ..., p\}$.
   \STATE $\mathcal{I}_1, ..., \mathcal{I}_{\frac{p}{b}} \gets$ partition $\pi$ into $\frac{p}{b}$ pieces.
   \STATE $w \gets 0_{p \times k}$.
   \STATE $R \gets 0_{n \times k}$
   \FOR{$\ell=1$ {\bfseries to} $n_e$}
       \STATE $\pi \gets$ random permutation of $\{1, ..., \frac{p}{b} \}$.
       \FOR{$i=1$ {\bfseries to} $\frac{p}{b}$}
           \STATE $I \gets \mathcal{I}_{\pi_i}$.
           \STATE $Z_b \gets \textsc{RandomFeaturesBlock}(X, I)$.
           \STATE $R \gets R - Z_b w(I, [k])$.
           \STATE Solve $(Z_b^\T Z_b + n\lambda I_b) w_b = Z_b^\T(Y - R)$.
           \STATE $R \gets R + Z_b w_b$.
           \STATE $w(I, [k]) \gets w_b$.
       \ENDFOR
   \ENDFOR
\end{algorithmic}
\end{algorithm}

}{
Our random features solver is the same as Algorithm~2 from
\cite{huang2014kernel}; its description is in the appendix.
}

\paragraph{Computation and communication overheads.}
Table~\ref{tab:comm_comp_costs} summarizes the computation and communication
costs of the algorithms presented below. The computation costs in the full
kernel are associated with computing the residual $R$ and solving a $b\times b$
linear system. For the \nystrom{} method (and random features), the computation
costs include computing $K_{I}^\T K_{I}$, $K_{I}^\T Y$ in parallel and
a similar local solve. Computing the gram matrix however requires adding $M$
matrices of size $b\times b$. Using a tree-based aggregation, this results in
$O(\log(M) b^2)$ bytes being transferred. We study how these costs
matter in practice in Section~\ref{sec:experiments}.

\paragraph{Computing the regularization path.}
Algorithms \ref{alg:kernel_bcd}, \ref{alg:nystrom_bcd}, and \ref{alg:rf_bcd}
are all described for a single input $\lambda$. In practice, for model
selection, one often computes an estimator for multiple values of $\lambda$. The
na{\"i}ve way of doing this is to run the algorithm again for each value of
$\lambda$. However, a faster approach, which we use in our experiments, is to
maintain separate models $\alpha_\lambda$ and seperate residuals $R_\lambda$
for each value of $\lambda$, and reuse the computation of the block matrices
$K_b$ for full kernel, $K_b^\T K_b$ for \nystrom{}, and $Z_b^\T Z_b$ for random
features. We can do this because the block matrices do not depend on the value
of $\lambda$.
As we show in Section~\ref{sec:experiments}, in each iteration of our algorithms,
a large fraction of time is spent in computing these block matrices; thus 
this optimization allows us compute solutions for 
multiple $\lambda$ values for essentially the price of a single solution.

We would like to note that in the case of \nystrom{} approximations,
\cite{rudi15} provides an algorithm for computing the regularization path
along $p$, the number of \nystrom{} samples, by using rank-one Cholesky updates.
We leave it as future work to see if a similar technique can be applied to our
\nystrom{} block coordinate algorithm.

\section{Optimization and statistical rates}
In this section we present our theoretical results which characterize
optimization error for kernel methods.
All proofs are deferred to the appendix.

\paragraph{Known convergence rates.}
We start by stating the existing rates for block coordinate descent.
To do this, we define a restricted Lipschitz constant as follows.
For any $Q(\cdot)$ such that $Q(x) \succeq 0$ for all $x$, define
\beqs
  L_{\max,b}(Q(\cdot)) := \sup_{x \in \R^{d}} \max_{\abs{I}=b} \lambda_{\max} (P_I Q(x) P_I) \:.
\eeqs
Standard analysis of block coordinate descent
(see e.g. Theorem~1 of \cite{wright15}) states that to reach accuracy
$\E f(\vct{w}^{\tau}) - f_* \leq \epsilon$, one requires at most
\beq
\tau \leq O\left( \frac{d L_{\max,b}}{b m} \log{\epsilon^{-1}} \right) \label{eq:bcdrate_basic}
\eeq
iterations, where $L_{\max,b} := L_{\max,b}(\nabla^2 f(\cdot))$.

While $L_{\max,b} \leq \sup_{x \in \R^{d}} \lambda_{\max}( \nabla^2 f(x))$ always, it
is easy to construct cases where the inequality is tight\footnote{Take, for
instance, any block diagonal matrix where the blocks are of size $b$.}.
In this case, the upper bound $\eqref{eq:bcdrate_basic}$ dictates that $d/b$
more iterations of block coordinate descent are needed to reach the same error
tolerance as the incremental gradient method.

\newcommand{\TabletextRates}{
Iteration complexity and block size requirements of
solving a full kernel system with block coordinate descent
versus \nystrom{} and random feature approximations.
For both \nystrom{}/RF, we assume that $p \gtrsim \log{n}$, and
for \nystrom{} we assume the regularized objective with $\gamma > 0$
(see Section~\ref{sec:background:kernelmethods}). Finally,
for both \nystrom{}/RF, the bounds hold w.h.p. over the feature sampling.
}

\begin{table*}[t]
\COND{}{
\caption{\TabletextRates{}}
\label{tab:rates}
}
\begin{center}
\begin{small}
\begin{sc}
\begin{tabular}{lcccc}
\hline
\abovespace\belowspace
 & \multicolumn{2}{c}{$\sigma_\ell(K)$ Exponential decay} & \multicolumn{2}{c}{$\sigma_\ell(K)$ Polynomial decay} \\
\hline
\abovespace\belowspace
Method& Iterations & Block Size & Iterations & Block Size  \\
\hline
\abovespace
Full    & $\widetilde{O}(n)$ & $\Omega(\log^2{n})$ & $\widetilde{O}(n^{\frac{2\beta}{2\beta+1}})$ & $\Omega(n^{1/(2\beta+1)} \log{n})$ \\
\nystrom{}  & $\widetilde{O}(np/\gamma)$ & $\Omega((1+\gamma)\log{n})$ & $\widetilde{O}(pn^{\frac{2\beta}{2\beta+1}}/\gamma)$ & $\Omega((1+\gamma)\log{n})$ \\
\belowspace
R.F. & $\widetilde{O}(n)$ & $\Omega(\log{n})$ & $\widetilde{O}(n^{\frac{2\beta}{2\beta+1}})$ & $\Omega(\log{n})$ \\
\hline
\end{tabular}
\end{sc}
\COND{
\caption{\TabletextRates{}}
\label{tab:rates}
}{}
\end{small}
\end{center}
\vspace{-0.2in}
\end{table*}

\subsection{Improved rate for quadratic functions}

In our experience, the case where block coordinate descent needs $d/b$ times
more iterations does not occur in practice.
To address this, we improve the analysis in the case of strongly convex and
smooth quadratic functions to depend only on a quantity which behaves like the
\emph{expected} value $\E_I \lambda_{\max}( P_I \nabla^2 f P_I )$ where $I$ is
drawn uniformly.

\begin{theorem}
\label{thm:bcdrate}
Let $f : \R^d \rightarrow \R$ be a quadratic function with
Hessian  $\nabla^2 f \in \R^{d \times d}$, and assume for some $L \geq m > 0$,
\beqs
 m \leq \lambda_{\min}(\nabla^2 f), \qquad \lambda_{\max}(\nabla^2 f) \leq L \:.
\eeqs
Let $\vct{w}^\tau$ denote the $\tau$-th
iterate of block coordinate descent with the index set $I_\tau$ consisting of $b
\in \{1,...,d\}$ indices drawn uniformly at random without replacement from
$\{1, ..., d\}$. The iterate $\vct{w}^\tau$ satisfies
\beqs
  \E f(w^\tau) - f_* \leq \left(1 - \frac{m}{2L_{\mathrm{eff}}} \right)^\tau (f(w^0) - f_*) \:,
\eeqs
where
\beqs
  L_{\mathrm{eff}} :=  e^2 L + \frac{d\log(2d^2/b)}{b} \norm{\diag(\nabla^2 f)}_{\infty} \:.
\eeqs
\end{theorem}

Theorem~\ref{thm:bcdrate} states that in order to reach an
$\epsilon$-sub-optimal solution for $f$, the number of iterations
required is at most
\begin{align}
	O\left( \left(\frac{L}{m} + \frac{1}{b} \frac{d\log{d}}{m} \norm{\diag(\nabla^2 f)}_\infty \right)\log{\epsilon^{-1}} \right) \:. \label{eq:bcdrate:iters}
\end{align}
That is, block coordinate descent pays the rate of gradient descent plus $1/b$
times the rate of standard ($b=1$) coordinate descent (ignoring log factors).
To see that this can be much better than the standard rate \eqref{eq:bcdrate_basic},
suppose that $d = p^2$ for some $p \geq 1$, and
consider any quadratic with Hessian
\begin{align*}
    \nabla^2 f = \lambda I_d + \diag( \ind_{\sqrt{d}} \ind_{\sqrt{d}}^\T, ..., \ind_{\sqrt{d}} \ind_{\sqrt{d}}^\T ) \in \R^{d \times d}\:,
\end{align*}
where $\ind_{\ell} \in \R^{\ell}$ is the all ones vector.
If we set $b = \sqrt{d}$,
the rate from \eqref{eq:bcdrate_basic} requires $\widetilde{O}(d/\lambda)$
iterations to reach tolerance $\epsilon$, whereas the rate from
Theorem~\ref{thm:bcdrate} requires only $\widetilde{O}(\sqrt{d}/\lambda)$ to
reach the same tolerance.

Equation~\eqref{eq:bcdrate:iters} suggests setting $b$ such that the second
term matches $L/m$ order wise. That is, as long as $b \gtrsim
d\log{d}\norm{\diag(\nabla^2 f)}_\infty / L$, we have that at most
$\widetilde{O}(L/m)$ iterations are necessary\footnote{
We use the notation $x \gtrsim y$ to mean there exists an absolute constant $C > 0$ such that $x \geq C y$,
and $\widetilde{O}(\cdot)$ to suppress dependence
on poly-logarithmic terms.
}.
In the sequel, we will assume this setting of $b$.

We highlight the main ideas behind the proof of Theorem~\ref{thm:bcdrate}.  The
proof proceeds in two steps. First, we establish a structural result which
states that, given a large set $\mathcal{G}$ of indices where the restricted
Lipschitz constant of the Hessian is well controlled, the overall dependence on
Lipschitz constant is not much worse than the maximum Lipschitz constant
restricted to $\mathcal{G}$.
Second, we use a probabilistic argument to show that such a set $\mathcal{G}$
does indeed exist.  The first result is based on a modification of the standard
coordinate descent proof, whereas the second result is based on a matrix
Chernoff argument.

\subsection{Rates for kernel optimization}
\label{sec:rates:kernelopt}

We now specialize Theorem~\ref{thm:bcdrate} to the optimization problems
described in Section~\ref{sec:background:kernelmethods}.  We assume the
asymptotic setting \cite{braun06} where $\sigma_\ell(K) = n \cdot \mu_\ell$ for
(a) exponential decay $\mu_\ell = e^{-\rho \ell}$ with $\rho > 0$ and (b)
polynomial decay $\mu_\ell = \ell^{-2\beta}$ with $\beta > 1/2$.  We also set
$\lambda$ to be the minimax optimal rate \cite{dicker15} for the settings of
(a) and (b): for exponential decay $\lambda = \log{n}/n$ and for polynomial
decay $\lambda = n^{-\frac{2\beta}{2\beta + 1}}$.  Finally, we assume that
$\sup_{x_1, x_2 \in \X} \kernel(x_1, x_2) \leq O(1)$.

Table~\ref{tab:rates} quantifies the iteration complexity of solving the full
kernel system versus the \nystrom{} and random features approximation.  Our
worst case analysis shows that the \nystrom{} system requires roughly
$p$ times more iterations to solve than random features. This difference is due
to the inability to reduce the \nystrom{} normal equation from quadratic in $K$
to linear in $K$, as is done in the full kernel normal equation.  Indeed, the \nystrom{}
method is less well conditioned in practice, and we observe similar phenomena in our
experiments below. The derivation of the bounds in Table~\ref{tab:rates} is deferred to
Appendix~\ref{sec:appendix:tabrates}.

\subsection{Primal versus dual coordinate methods}
\label{sec:rates:primaldual}
Duality gives us a choice as to whether to solve the primal or dual
problem; strong duality asserts that both solutions are equivalent. We can use
this freedom to our advantage, picking the formulation which yields the most
numerically stable system.  For instance, in the full kernel solver 
we chose to work with the system $(K+n\lambda I_n)\alpha = Y$ instead of $K(K+n\lambda I_n)\alpha = KY$.
The former is actually the dual system, and the latter is the primal.  Here, the primal
system has a condition number which is roughly the square of the dual.

On the other hand, for both \nystrom{} and random features, our system works on the
primal formulation.  This is intuitively desirable since $p \ll n$
and hence the primal system is much smaller.  However, some authors
including \cite{shalevshwartz13} advocate for the dual
formulation even when $p \ll n$. We claim that, at least in the case
of random Fourier features, their argument does not apply.

To do this, we consider the random features program with $b = k = 1$, which
fits the framework of \cite{shalevshwartz13} the closest.  By the primal-dual
correspondence $w = \frac{1}{n\lambda} Z^\T \alpha$, the dual program is
\beqs
	\max_{\alpha \in \R^n} \frac{1}{n} Y^\T \alpha - \frac{1}{n} \norm{\alpha}^2 - \frac{1}{\lambda n^2} \alpha^\T ZZ^\T \alpha \:.
\eeqs
Theorem~5 from \cite{shalevshwartz13} states that $O((n +
L_{\max,1}(ZZ^\T)/\lambda)\log{\epsilon^{-1}})$ iterations of dual coordinate
ascent are sufficient to reach an $\epsilon$-sub-optimal primal solution. On
the other hand, Equation~\eqref{eq:bcdrate_basic} yields that at most $O((p
L_{\max,1}(Z^\T Z)/n\lambda) \log{\epsilon^{-1}})$ iterations of primal
coordinate descent are sufficient to reach the same accuracy.

For random Fourier features, both $L_{\max,1}(ZZ^\T)$ and $L_{\max,1}(Z^\T Z)$ can be
easily upper bounded,
since $\abs{\ip{z(x_i)}{z(x_i)}} \leq \frac{2}{p} \sum_{k=1}^{p} \abs{\cos(x_i^\T w_k
+ b_k)} \leq 2$ and also
$\norm{\diag(Z^\T Z)}_\infty = \max_{1 \leq k \leq p} \frac{2}{p} \sum_{i=1}^{n} \cos^2(w_k^\T x_i + b_k) \leq \frac{2n}{p}$.
Therefore, the dual rate is $\widetilde{O}(n + 1/\lambda)$ and the primal rate
is $\widetilde{O}(1/\lambda)$. That is, for random Fourier features,
the primal rate upper bound beats the dual rate upper bound.


\newcommand{\TabletextDatasets}{
Datasets used for evaluation. Here $n$, $d$, $k$ refer to the number of training examples, features and
classes respectively. Size represents the size of the full kernel matrix in terabytes.
}

\begin{table}[t]
\COND{}{
\caption{\TabletextDatasets{}}
\label{tab:datasets}
}
\begin{center}
\begin{small}
\begin{sc}
\begin{tabular}{lcccc}
\hline
\abovespace\belowspace
Dataset & $n$ & $d$ & $k$ & Size (TB)\\ 
\hline
\abovespace
TIMIT    & $2,251,569$ & $440$ & $147$ & $40.56$ \\ 
Yelp     & $1,255,412$ & $65,282,968$ & $5$ &$12.61$ \\ 
\belowspace
CIFAR-10 & $500,000$ & $4096$ & $10$ & $2.00$ \\ 
\hline
\end{tabular}
\end{sc}
\end{small}
\COND{
\caption{\TabletextDatasets{}}
\label{tab:datasets}
}{}
\end{center}
\COND{}{\vspace{-0.2in}}
\end{table}

\section{Experiments}
\label{sec:experiments}

\begin{figure*}[t!]
  \centering
  \subfigure[]{\label{fig:timit-compare}\includegraphics[width=0.32\textwidth]{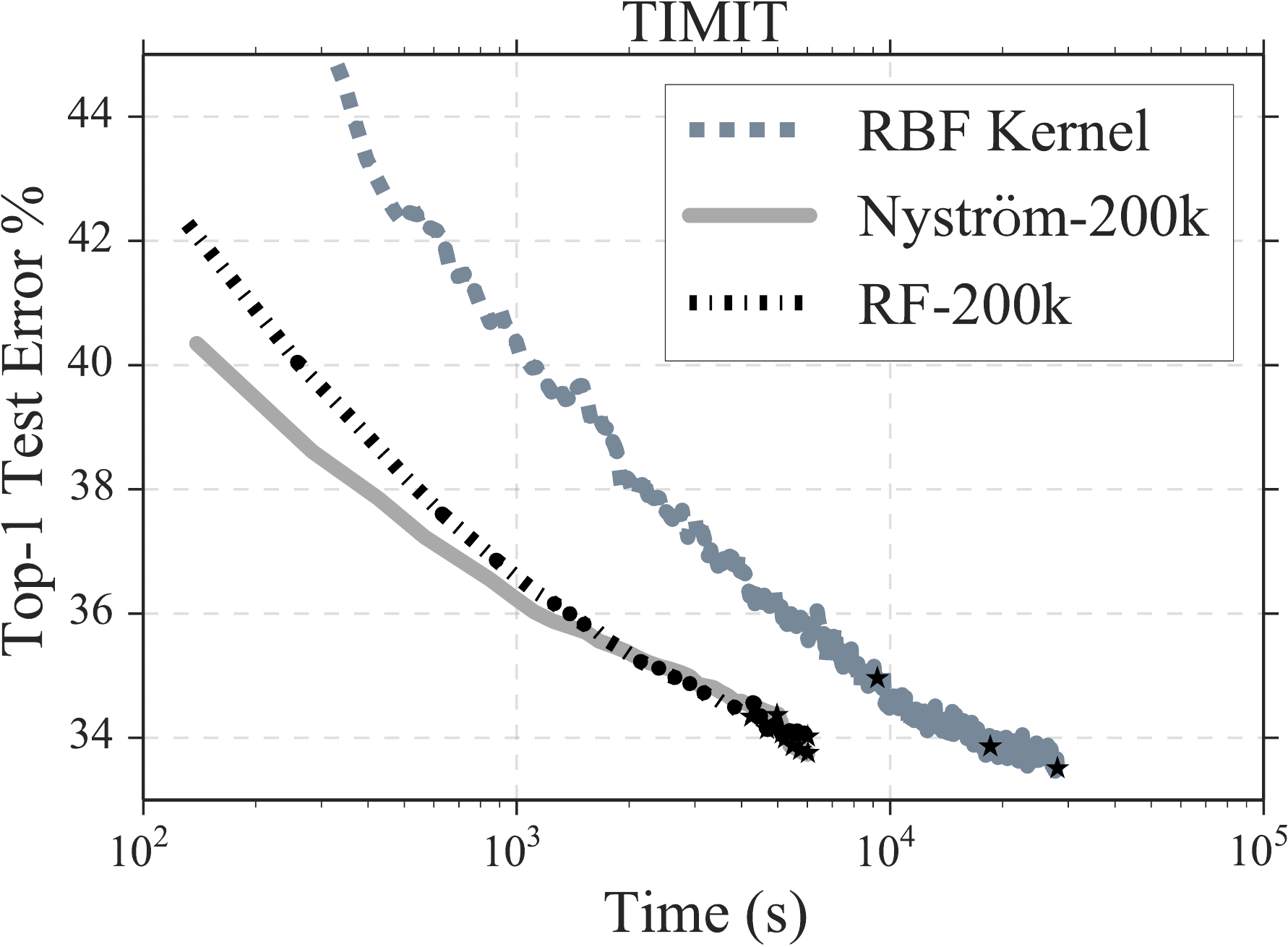}}
  \subfigure[]{\label{fig:yelp-compare}\includegraphics[width=0.32\textwidth]{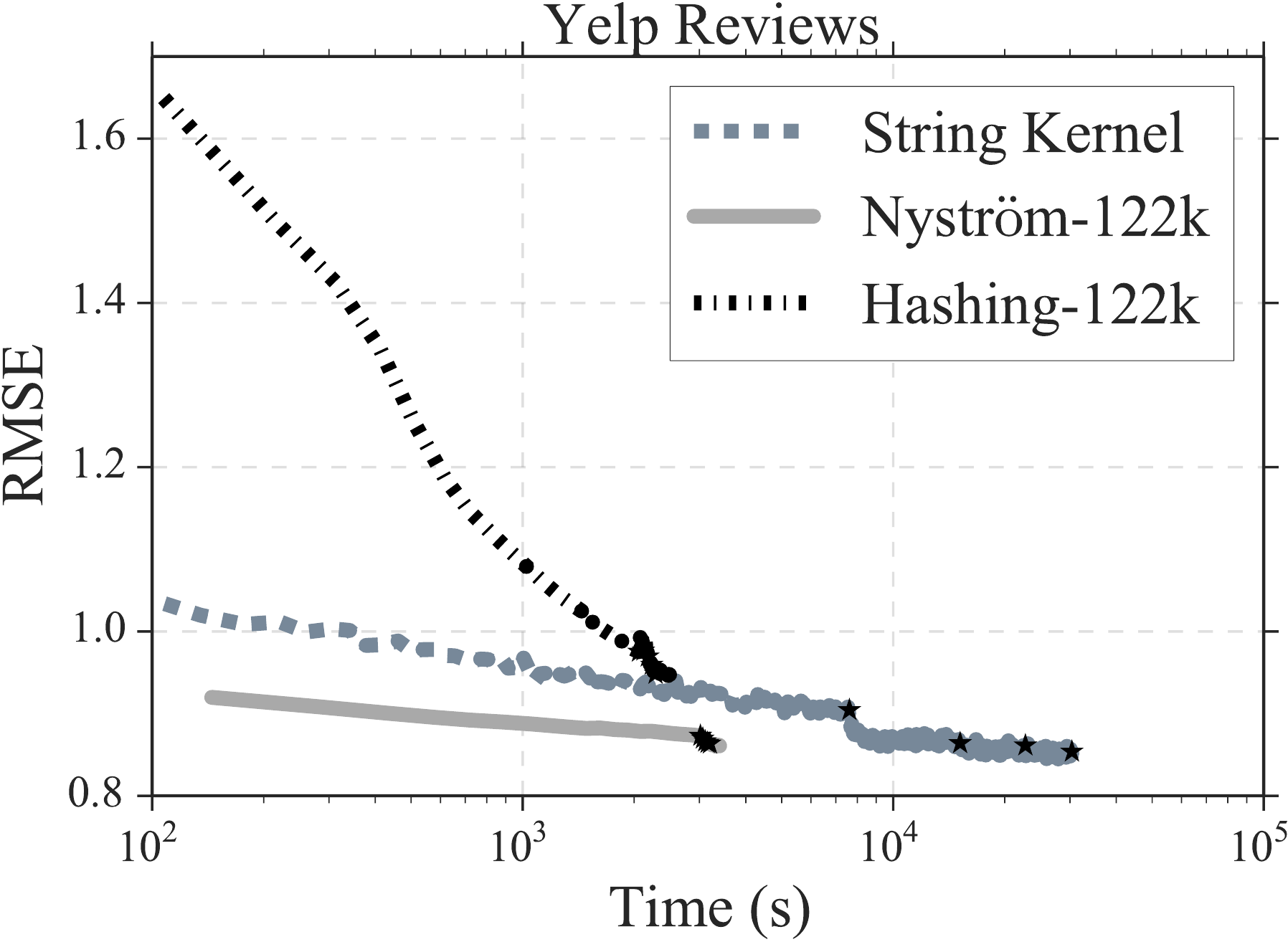}}
  \subfigure[]{\label{fig:cifar-compare}\includegraphics[width=0.32\textwidth]{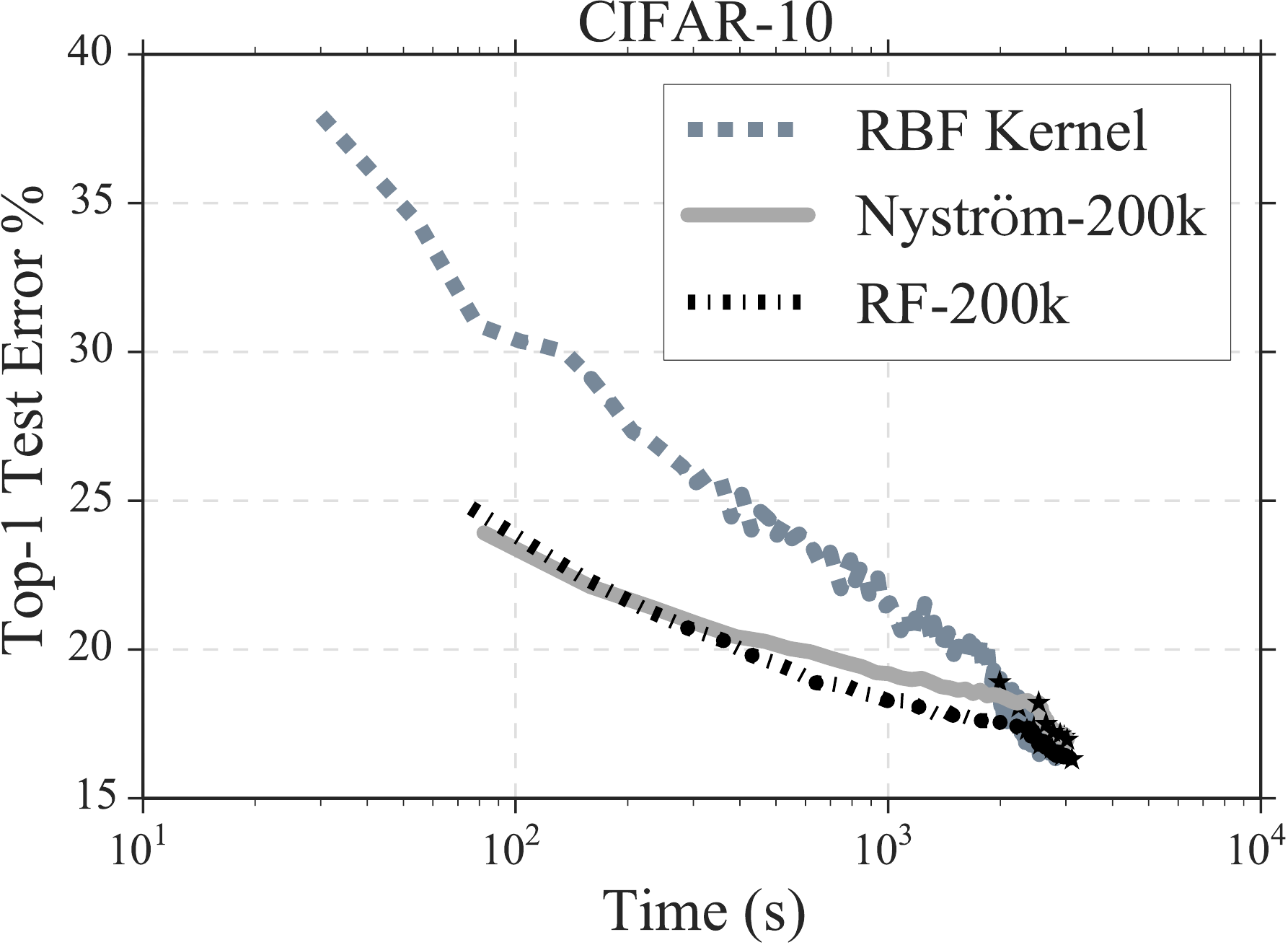}}
  \vspace{-0.2in}
  \caption{Comparison of classification error using different methods on the TIMIT, Yelp, and CIFAR-10
datasets. We measure the test error after every block of the algorithm; black stars
denote the end of an epoch.}
  \label{fig:time-compare}
  \COND{}{\vspace{-0.1in}}
\end{figure*}

\begin{figure*}[t!]
  \centering
  \subfigure[]{\label{fig:timit-feats}\includegraphics[width=0.32\textwidth]{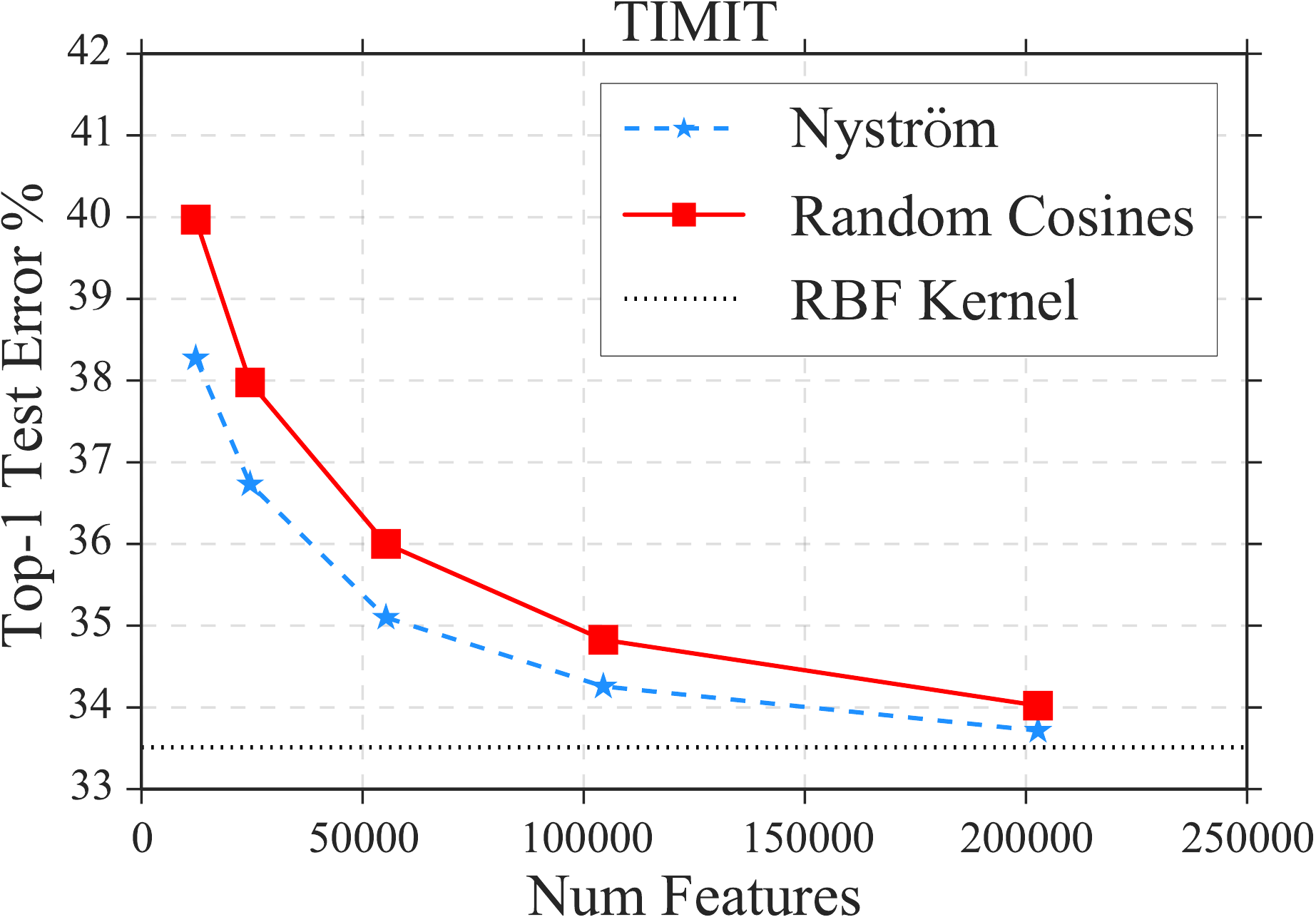}}
  \subfigure[]{\label{fig:yelp-feats}\includegraphics[width=0.32\textwidth]{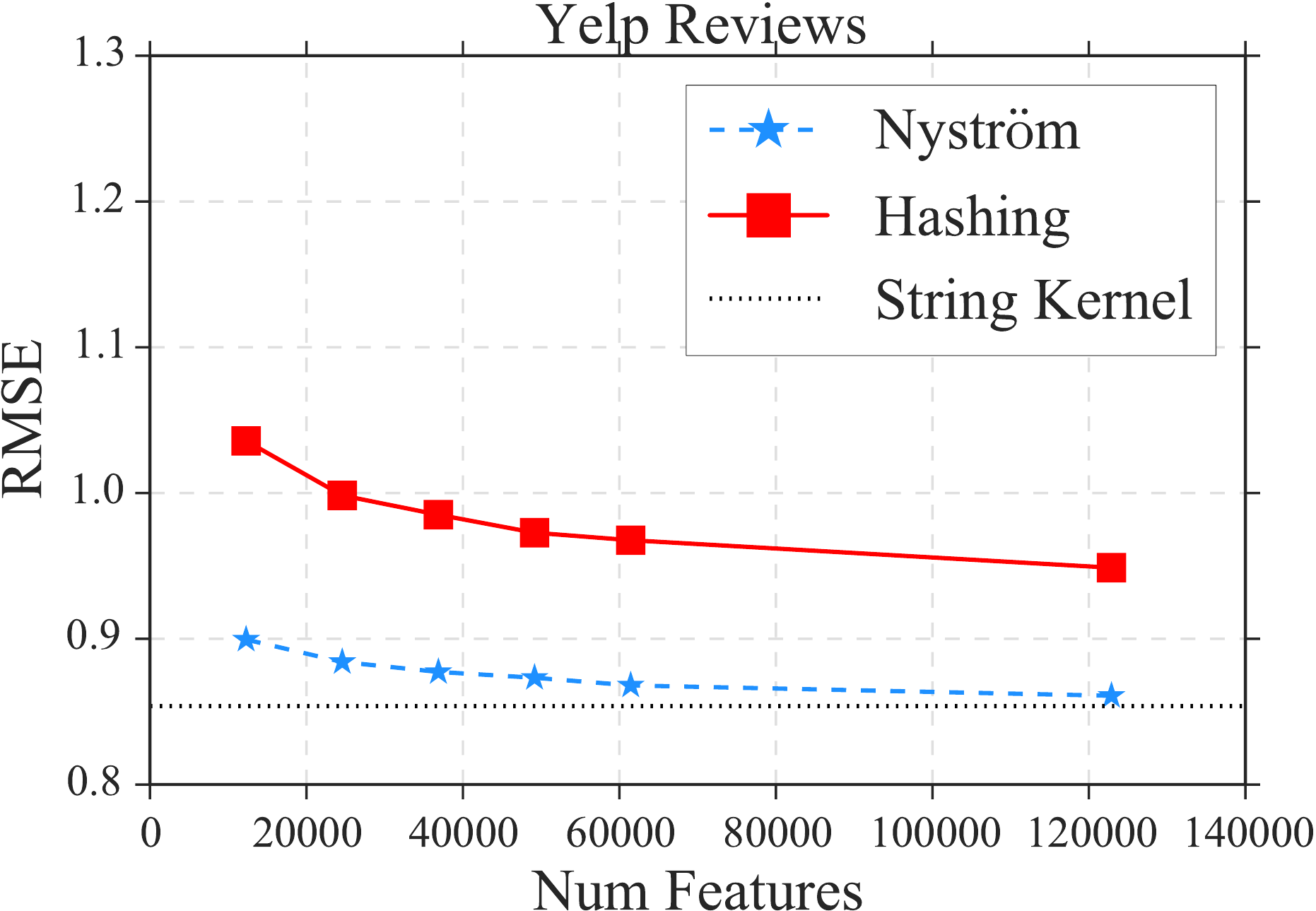}}
  \subfigure[]{\label{fig:cifar-feats}\includegraphics[width=0.32\textwidth]{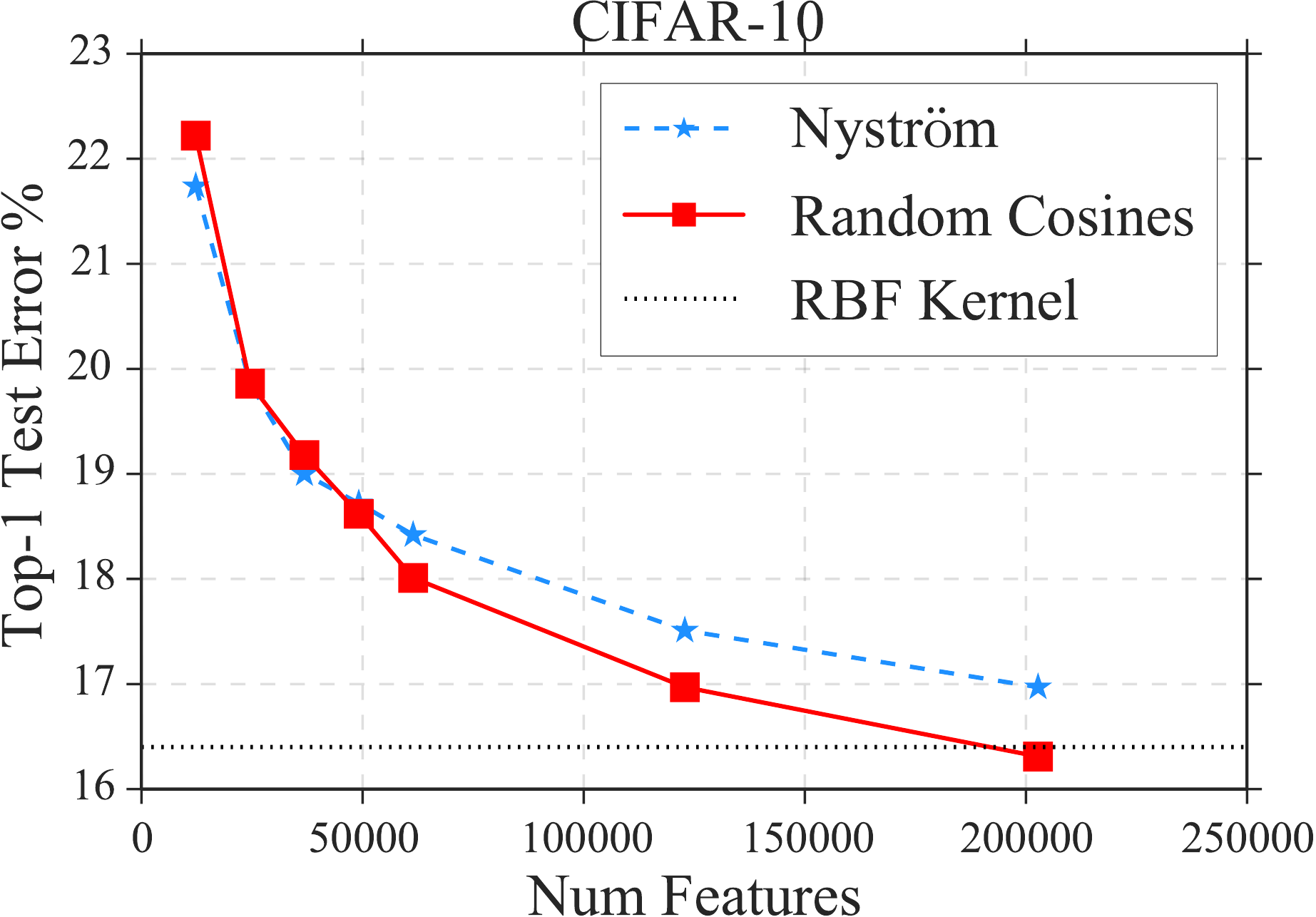}}
  \COND{}{\vspace{-0.2in}}
  \caption{Classification error as we increase the number of features for \nystrom{}, Random Features on the TIMIT, Yelp, and CIFAR-10.}
  \label{fig:feats-compare}
  \COND{}{\vspace{-0.2in}}
\end{figure*}

This section describes our experimental evaluation.
We implement our algorithms in Scala on top of Apache Spark~\cite{zaharia12}.
Our experiments are run on Amazon EC2, with a cluster of 128
\texttt{r3.2xlarge} machines, each of which has 4 physical cores and 62 GB of RAM.

We measure classification accuracy for three large datasets
spanning speech (TIMIT), text (Yelp), and vision (CIFAR-10). The size of
these datasets are summarized in Table~\ref{tab:datasets}.
For all our experiments, we set the block size to $b=6144$. We shuffle the raw
data at the beginning of the algorithm, and select blocks in a random order for
block coordinate descent. For the \nystrom{} method, we
uniformly sample $p$ columns without replacement from the full kernel matrix.

\subsection{TIMIT}
We evaluate a phone classification task on the TIMIT dataset\footnote{\url{https://catalog.ldc.upenn.edu/LDC93S1}}, which
consists of spoken audio from 462 speakers.
We use the same preprocessing pipeline as \cite{huang2014kernel},
resulting in $2.25 \times 10^6$ training examples and $10^5$ test examples. The preprocessing
pipeline produces a dense vector with $440$ features and we use a shuffled version of this as the input to our
kernel methods. We apply a Gaussian (RBF) kernel for the \nystrom{} and exact methods and use
random cosines~\cite{rahimi07} for the random feature method. Figure~\ref{fig:timit-compare} shows the
top-1 test error for each technique. From the figure, we can see that while the exact method takes the longest
to complete a full epoch (around $2.5$ hours), it achieves the lowest top-1 test-error ($33.51\%$) among all
methods after $3$ epochs. Furthermore, unlike the exact method, the data for the \nystrom{} and random features
with $p=200,000$ can be cached in memory; as a result, the approximate methods run much faster after
the first epoch compared to the exact method.

We also compare \nystrom{} and random features by varying
$p$ in Figure~\ref{fig:timit-feats} and find that for $p \geq 100,000$ both methods approach the
test error of the full kernel within $1\%$.

\subsection{Yelp Reviews}
We next evaluate a text classification task where the goal is to predict a rating from one to five stars from the
text of a review.  The data comes from
Yelp's academic dataset\footnote{\url{https://www.yelp.com/academic_dataset}}, which
consists of $1.5 \times 10^6$ customer reviews. We set aside 20\% of the reviews
for test, and train on the remaining 80\%. For preprocessing, we use
\texttt{nltk}\footnote{\url{http://www.nltk.org/}} for tokenizing and stemming
documents. We then remove English stop words and create $3$-grams, resulting in
a sparse vector with dimension $6.52 \times 10^7$.
For the exact and \nystrom{} experiments, we apply a linear kernel,
which when combined with the $3$-grams can be viewed as an instance of a
string kernel \cite{sonnenburg07}.
For random features, we apply a hash kernel \cite{weinberger09} using
\texttt{MurmurHash3} as our hash function. Since we are predicting ratings for a review, we
measure accuracy by using the root mean square error (RMSE) of the predicted rating
as compared to the actual rating. Figure~\ref{fig:yelp-compare} shows how various kernel methods perform
with respect to wall clock time. From the figure,  we can see that the string
kernel performs much better than the hash-based random features for this
classification task. We also see that the \nystrom{} method achieves almost the
same RMSE ($0.861$) as the full kernel ($0.854$) when using $122,000$ features.
Finally, Figure~\ref{fig:yelp-feats} shows that the improved accuracy from using
the string kernel over hashing holds as we vary the number of features ($p$) for the \nystrom{} and random feature methods.

\begin{figure*}[t!]
  \centering
  \hspace{0.2in}
  \includegraphics[width=0.6\textwidth]{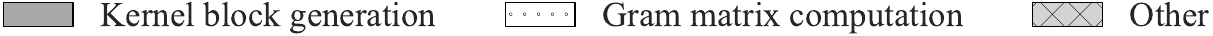}
  \vspace{-0.2in}
\end{figure*}
\begin{figure*}[t!]
  \centering
  \subfigure[]{\label{fig:timit-breakdown}\includegraphics[width=0.32\textwidth]{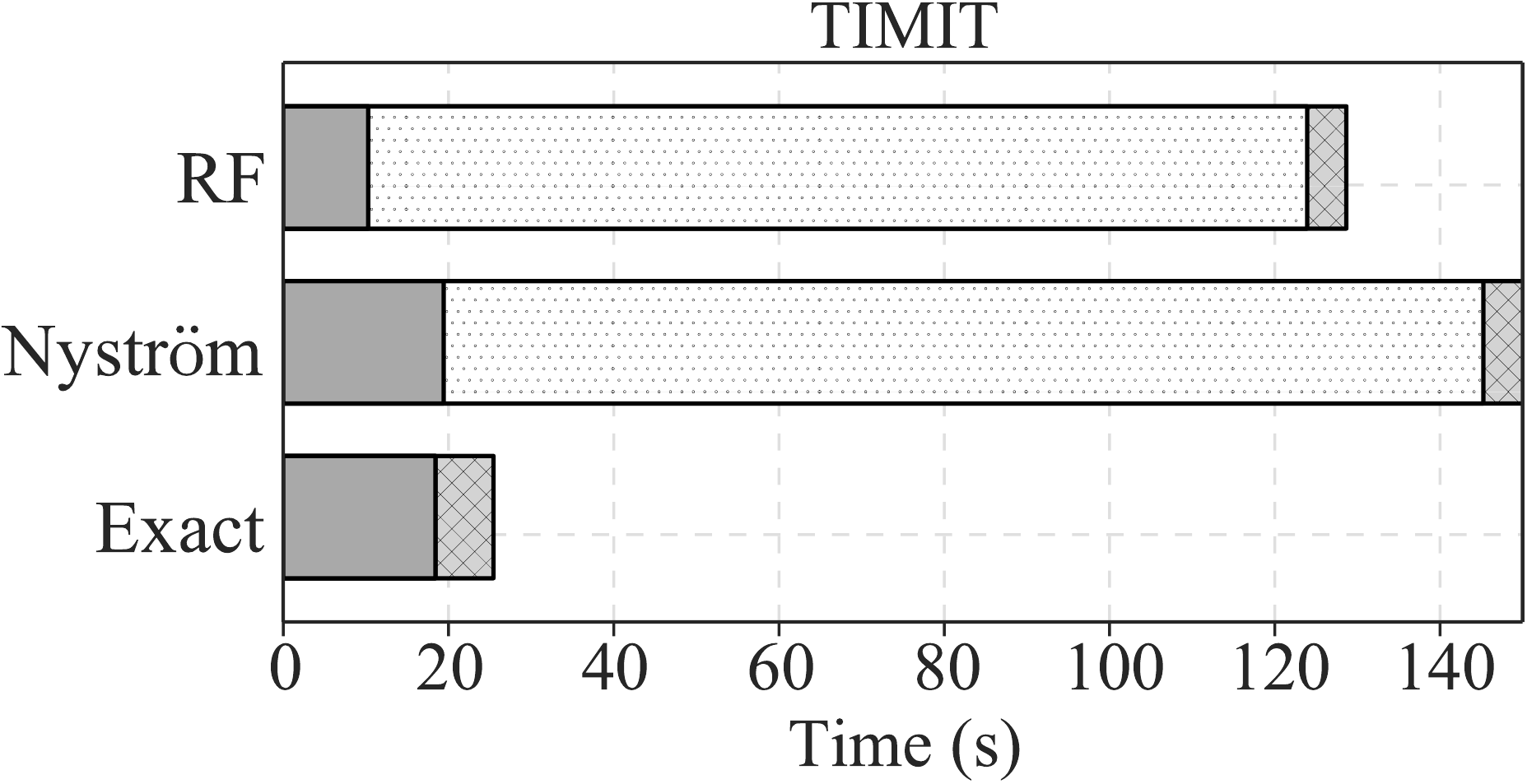}}
  \subfigure[]{\label{fig:yelp-breakdown}\includegraphics[width=0.32\textwidth]{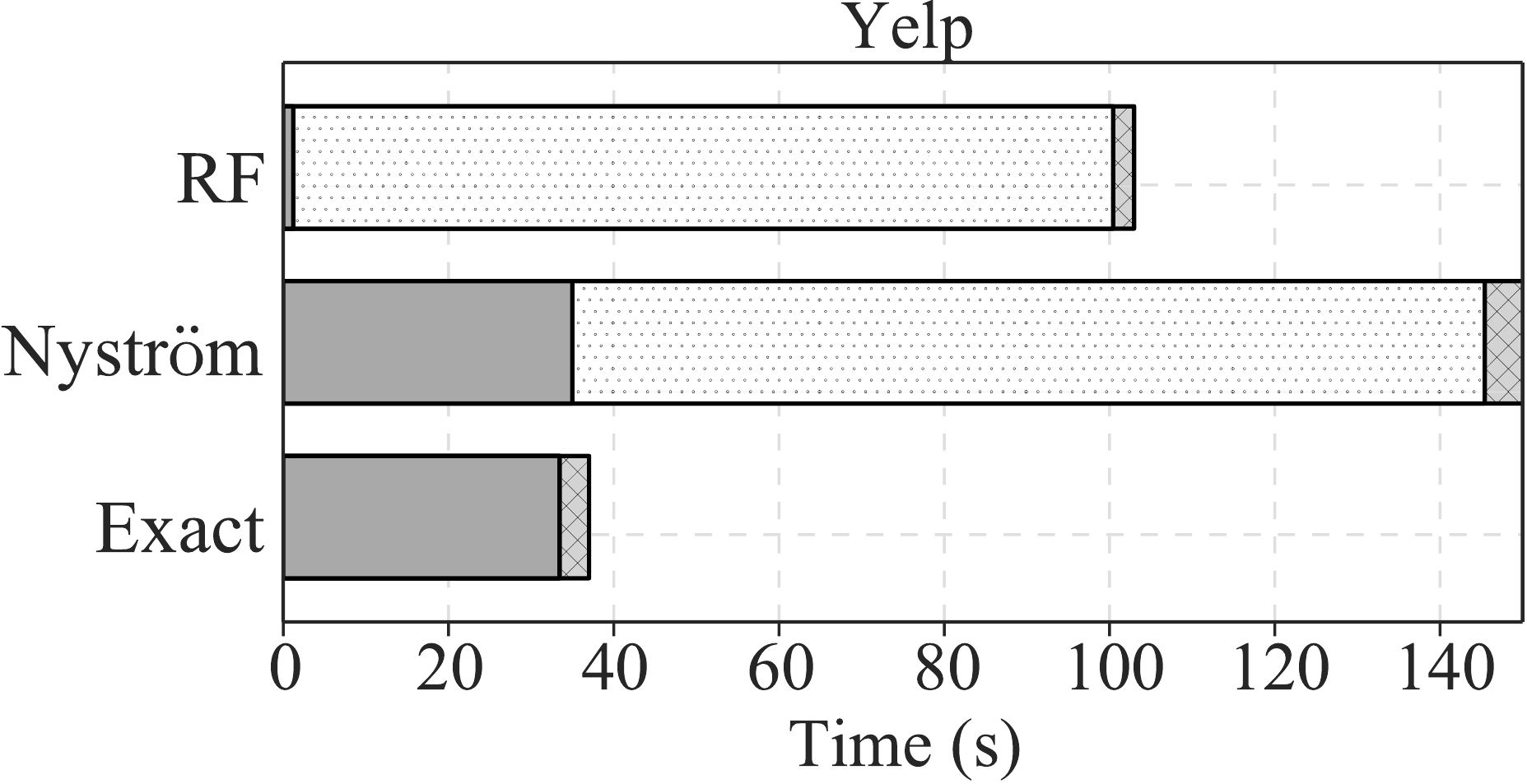}}
  \subfigure[]{\label{fig:cifar-breakdown}\includegraphics[width=0.32\textwidth]{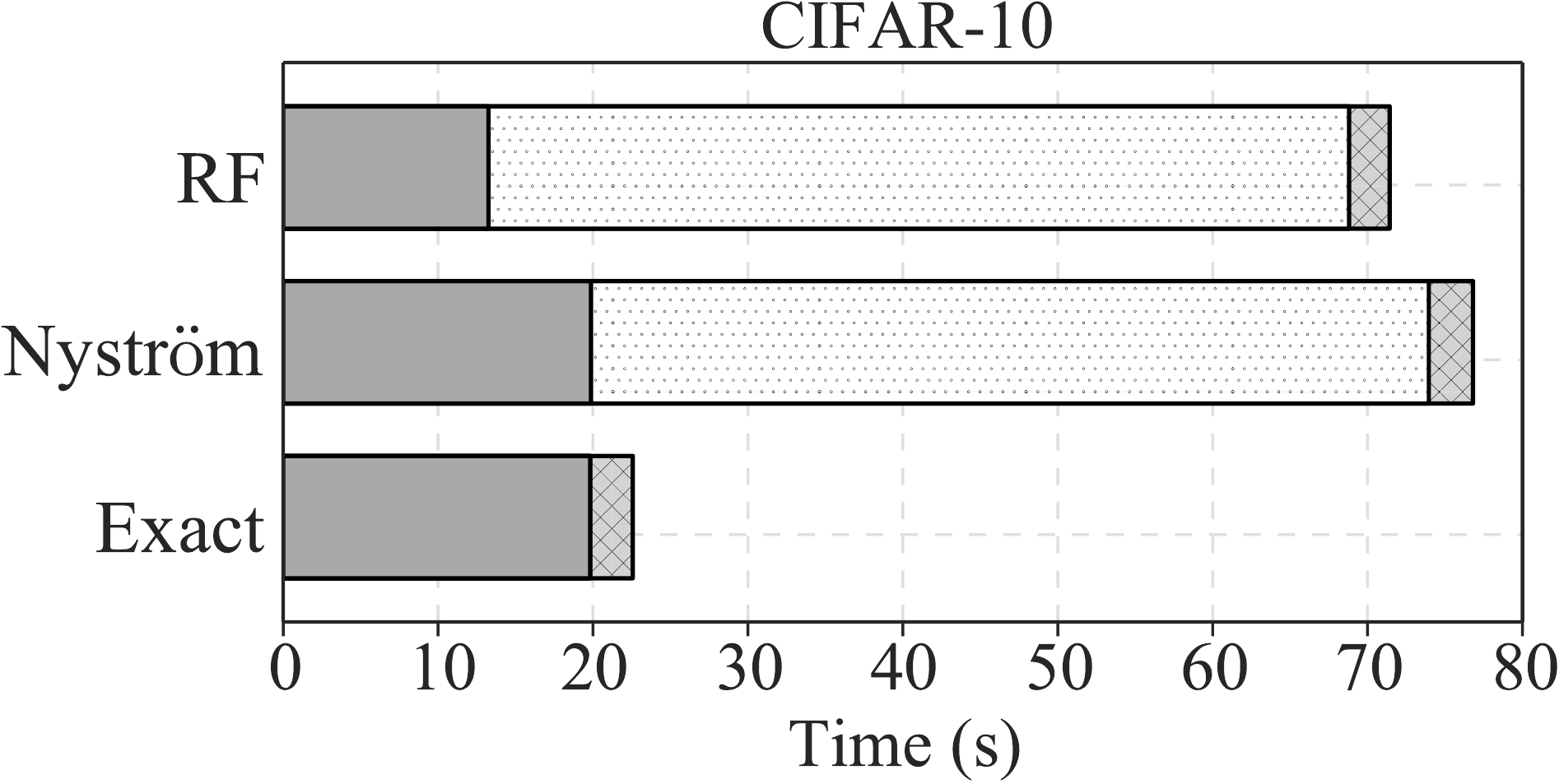}}
  \COND{}{\vspace{-0.2in}}
  \caption{Breakdown of time to compute a \emph{single} block in the first epoch on the TIMIT, CIFAR-10,
and Yelp datasets.}
  \label{fig:break-it-down}
  \COND{}{\vspace{-0.1in}}
\end{figure*}

\subsection{CIFAR-10}
Our last task involves image classification for the CIFAR-10 dataset \footnote{\url{cs.toronto.edu/~kriz/cifar.html}}.
We perform the same data augmentation as described in \texttt{cuda-convnet2}
\footnote{\url{github.com/akrizhevsky/cuda-convnet2}}, which
results in 500,000 train images.
For preprocessing, we use a pipeline similar to \cite{coates12}, replacing the
$k$-means step with random image patches. Using 512 random image patches,
we get $4096$ features per image and fitting a linear model with these features
gives us $25.7\%$ test error.  For our kernel methods, we start with these $4096$
features as the input and we use the RBF kernel for the exact and \nystrom{} method and
random cosines for the random features method.

From Figure~\ref{fig:cifar-compare}, we see that on CIFAR-10 the full kernel takes around the same
time as \nystrom{} and random features. This is because we have fewer examples
($n=500,000$) and this leads to fewer blocks that need to be solved per-epoch.  We are also able
to cache the entire kernel matrix in memory ($\sim$ 2TB) in this case and this provides a speedup
after the first epoch.

Furthermore, as shown in Figure~\ref{fig:cifar-feats}, we see that applying a non-linear kernel to the
output of convolutions using random patches can result in significant improvement in accuracy. With the non-linear kernel, we
achieve a test error of $16.4\%$, which is $9.3\%$ lower than a linear model trained with
the same features.

\COND{
  \begin{figure*}[t!]
    \centering
    \subfigure[]{\label{fig:cifar-convergence}\includegraphics[width=0.47\textwidth]{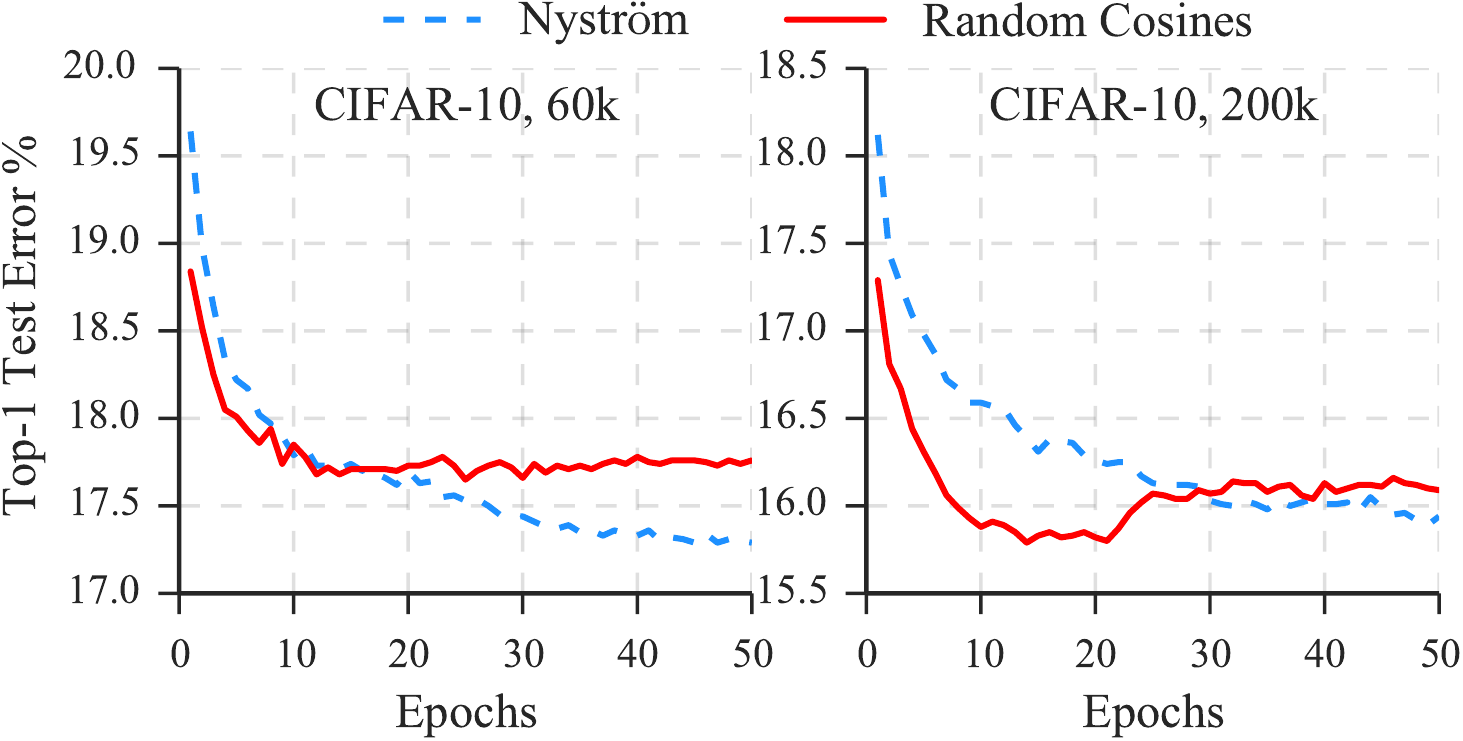}}
    \subfigure[]{\label{fig:cifar-feats-50epochs}\includegraphics[width=0.48\textwidth]{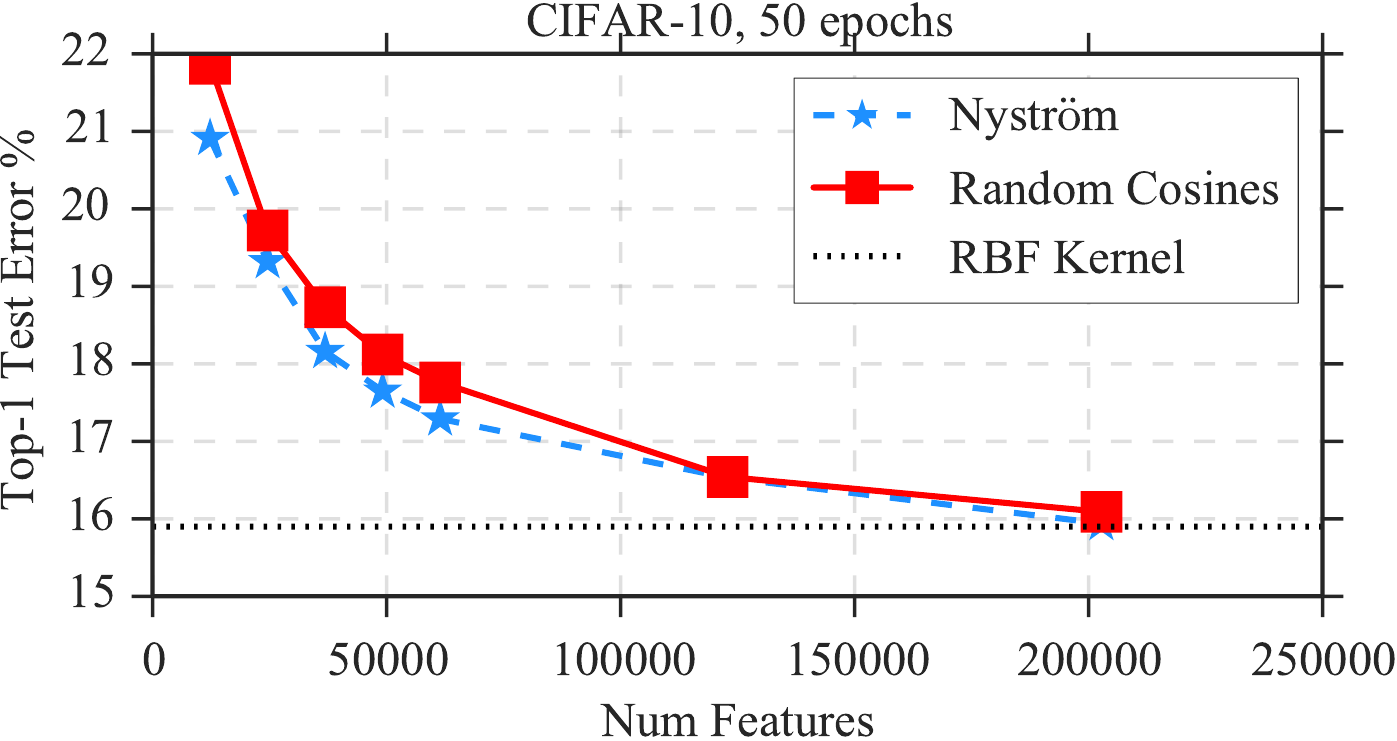}}
    \label{fig:cifar-convergence-50}
    \caption{Convergence rate, Top-1 test error for CIFAR-10 across 50 epochs for \nystrom{} and random features.}
  \end{figure*}
}{
\begin{figure}[t!]
  \centering
  \includegraphics[width=0.95\columnwidth]{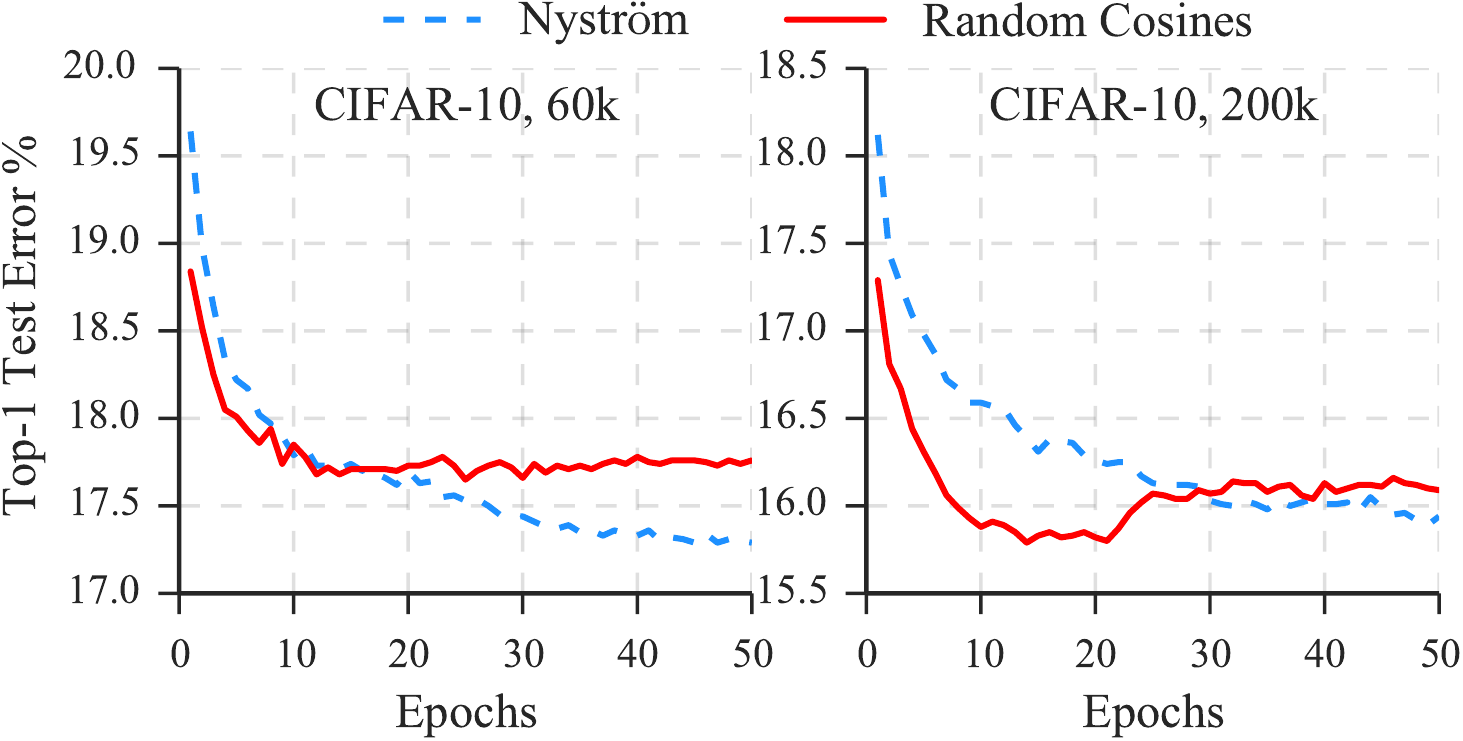}
  \vspace{-0.1in}
  \caption{Comparing convergence of \nystrom{} and random features for CIFAR-10.}
  \label{fig:cifar-convergence}
  \vspace{-0.1in}
\end{figure}
}
When comparing random features and \nystrom{} after $5$ epochs for various
values of $p$, we see that they perform similarly for smaller number of
features but that random features performs better with larger number of
features. We believe that this is due to the \nystrom{} normal equations having
a larger condition number for the CIFAR-10 augmented dataset which leads to a
worse convergence rate.  We verify this in Figures~\ref{fig:cifar-convergence}
and \ref{fig:cifar-feats-50epochs} by running by \nystrom{} and random feature
solvers for 50 epochs. In Figure~\ref{fig:cifar-convergence}, we fix the number
of random features to $p \in \{60,000, 200,000\}$, and we see that \nystrom{}
takes more epochs to converge but reaches a better test error.  In
Figure~\ref{fig:cifar-feats-50epochs}, we perform the same sweep as in
Figure~\ref{fig:cifar-feats} except we stop at 50 epochs instead of 5. Indeed,
when we do this, the difference between \nystrom{} and random features matches
the trends in Figures~\ref{fig:timit-feats} and \ref{fig:yelp-feats}.

\subsection{Performance}
\label{sec:eval:perf}
We next study the runtime performance characteristics of each method. Figure~\ref{fig:break-it-down}
shows a timing breakdown for running \emph{one block} of block coordinate descent on the
three datasets.  From the figure, we see that the choice of the kernel approximation can significantly impact
performance since different kernels take different amounts of time to generate. For example, the hash random feature used for the Yelp dataset is much cheaper to compute than
the string kernel. However, computing a block of the RBF kernel is similar in cost to computing a
block of random cosine features. This results in similar performance characteristics for the \nystrom{} and
random feature methods on TIMIT. 

We also observe that the full kernel takes the least amount of time to solve one block.
This is primarily because the full kernel does not compute a gram
matrix $Z_b^{\T}Z_b$ and only extracts a block of the kernel matrix $K_{bb}$. Thus, when the number
of blocks is small, as is the case for CIFAR-10 in Figure~\ref{fig:cifar-compare}, the full
kernel's performance becomes comparable to the \nystrom{} method. 

\COND{}{
\begin{figure}[t!]
  \centering
  \includegraphics[width=0.90\columnwidth]{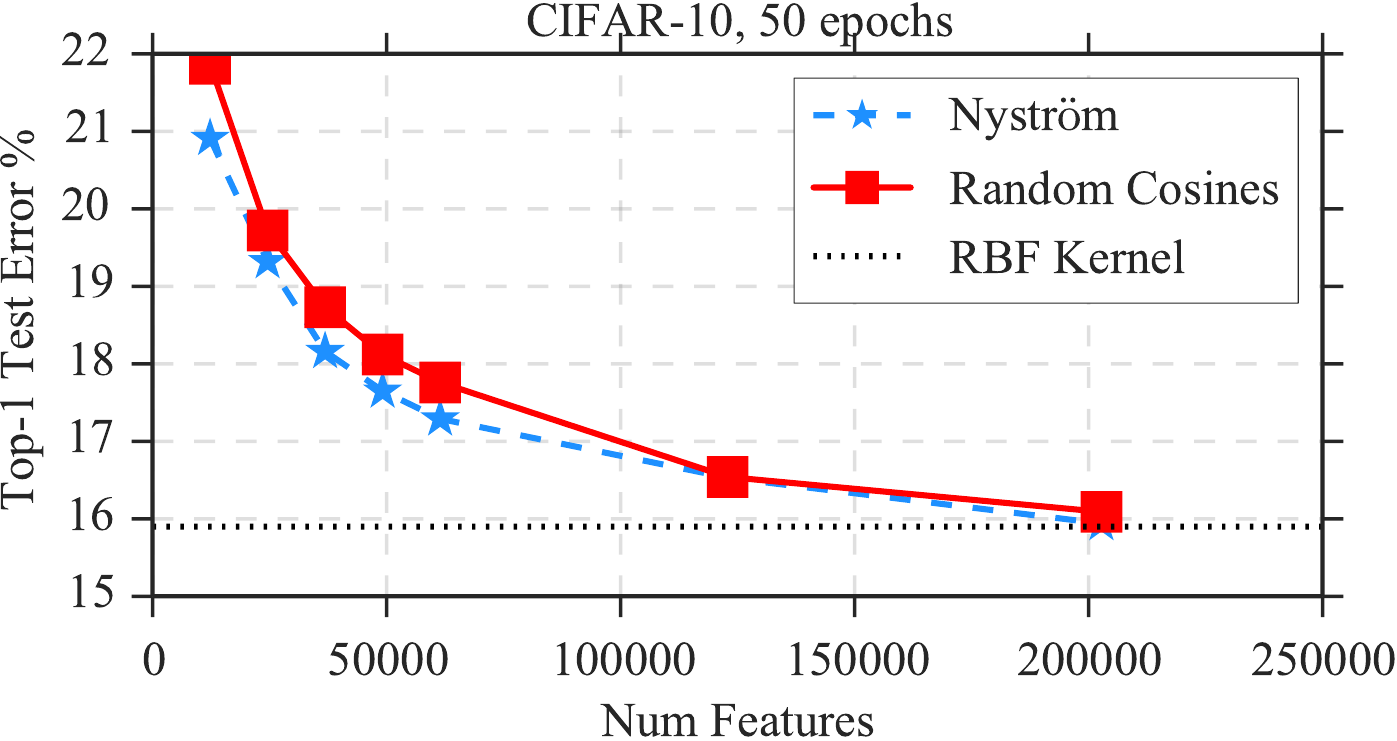}
  \vspace{-0.1in}
  \caption{CIFAR-10 Top-1 test error versus number of features after 50 epochs.}
  \label{fig:cifar-feats-50epochs}
  \vspace{-0.1in}
\end{figure}
}

\COND{}{
Figure~\ref{fig:break-it-down} also demonstrates that
computing the gram matrix and generating the kernel block are the two most expensive steps in our
algorithm. Computing the gram matrix uses distributed matrix multiplication, which is
well studied~\cite{van1997summa}.
In practice, we also observe that
generating the RBF kernel scales well; our scaling results are included in Appendix~\ref{sec:appendix:scaling}
for completeness. Thus, we believe that our algorithms will scale well as datasets and clusters grow in size.
}

\COND{
  \subsection{Scalability of RBF kernel generation}
  \label{sec:appendix:scaling}

  \begin{figure*}[t!]
    \centering
    \subfigure[]{\label{fig:timit-rbf-scale}\includegraphics[width=0.45\textwidth]{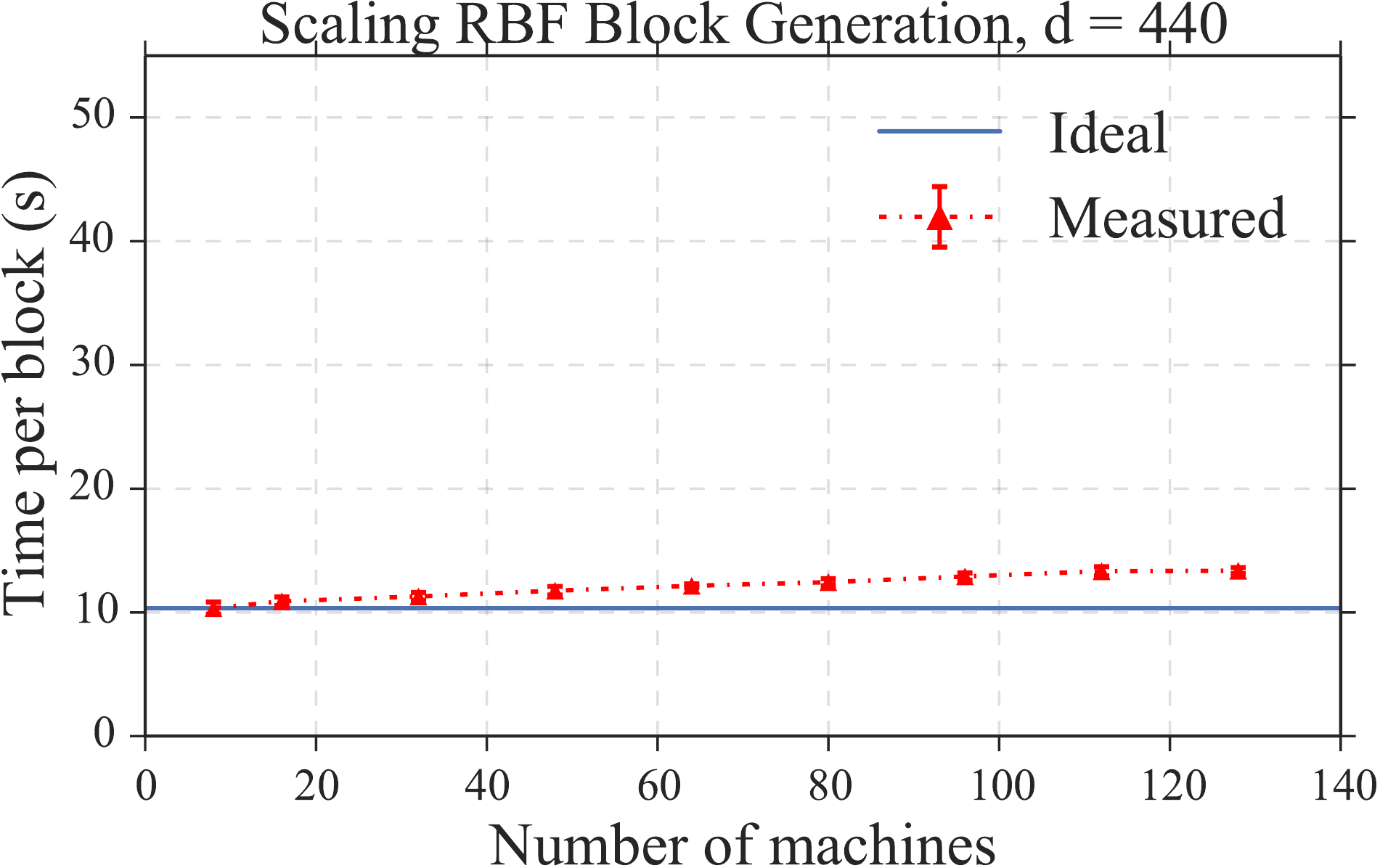}}
    \subfigure[]{\label{fig:cifar-rbf-scale}\includegraphics[width=0.45\textwidth]{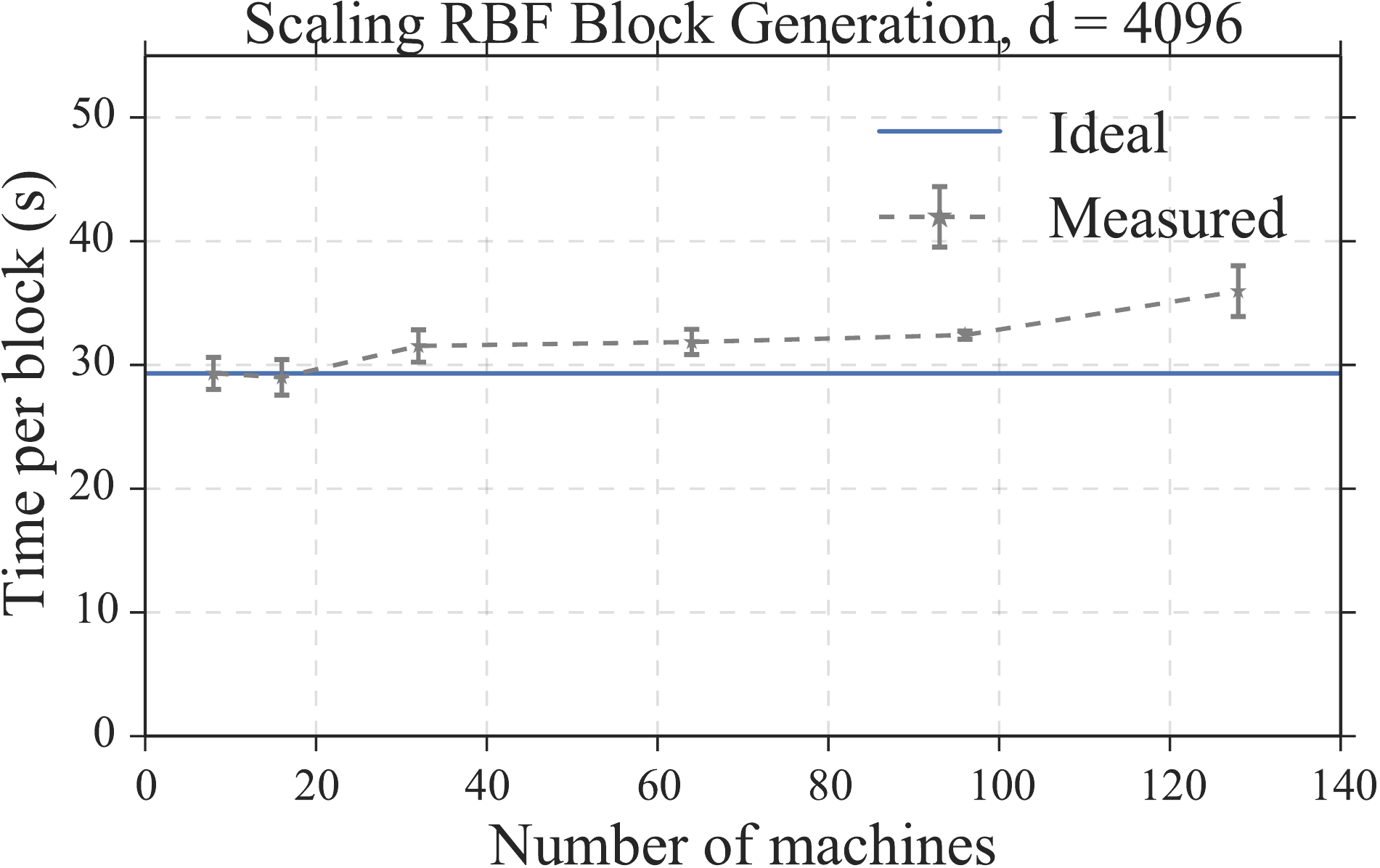}}
    \caption{Time taken to compute one block of the RBF kernel as we scale the number of examples and
  the number of machines used.}
    \label{fig:scaling-rbf}
  \end{figure*}

  Figure~\ref{fig:break-it-down} also demonstrates that
  computing the gram matrix and generating the kernel block are the two most expensive steps in our
  algorithm. Computing the gram matrix uses distributed matrix multiplication, which is
  well studied~\cite{van1997summa}.
  To see how the cost of kernel generation changes as dataset size grows, we perform a weak scaling experiment
  where we increase the number of examples and the number of machines used while keeping the number of
  examples per machine constant ($n=16,384$). We run this experiment for $d=440$ and $d=4096$, which are the
  number of features in TIMIT and CIFAR-10 respectively. Figure~\ref{fig:scaling-rbf} contains results
  from this experiment. In the weak scaling scenario, ideal scaling implies that the time to
  generate a block of the kernel matrix remains constant as we increase both the data and the number of
  machines. However, computing a block of the RBF kernel
  involves broadcasting a $b \times d$ matrix to all the machines in the cluster.
  This causes a slight decrease in performance as we go from $8$ to $128$ machines. As broadcast
  routines scale as $O(\log{M})$, we believe that our kernel block generation methods will continue to
  scale well for larger datasets.
}{}

\section{Conclusion}
This paper shows that scalable kernel machines are feasible with
distributed computation.
There are several theoretical and experimental continuations of this work.

On the theoretical side, a limitation of our current analysis of block coordinate 
descent is that we cannot
hope to achieve rates better than
gradient descent. We believe it is possible to leverage the direct solve in
\eqref{eq:bcd_direct_solve_update} to improve our rate,  since when $b = d$ the
algorithm reduces to Newton's method.
We are also interested in seeing if acceleration techniques can be
applied to 
substantially reduce the number of
iterations needed. 

On the experimental side, we would like to extend our algorithm to handle other
losses than the square loss; ADMM might be one approach for this.
More broadly, since solving a least squares program is a core
primitive for many optimization algorithms, we are interested to see if our
techniques can be applied in other domains.

\section*{Acknowledgements}
The authors thank Vikas Sindhwani and the IBM corporation for providing access
to the derived TIMIT dataset used in our experiments.  BR is generously
supported by ONR awards  N00014-14-1-0024, N00014-15-1-2620, and
N00014-13-1-0129, and NSF awards CCF-1148243 and CCF-1217058. RR is supported by
the U.S. Department of Energy under award numbers DE-SC0008700 and AC02-05CH11231. 
This research is supported in part by NSF CISE Expeditions Award CCF-1139158, LBNL Award
7076018, DARPA XData Award FA8750-12-2-0331, and gifts from Amazon Web
Services, Google, SAP, The Thomas and Stacey Siebel Foundation, Adatao, Adobe,
Apple, Inc., Blue Goji, Bosch, C3Energy, Cisco, Cray, Cloudera, EMC2, Ericsson,
Facebook, Guavus, HP, Huawei, Informatica, Intel, Microsoft, NetApp, Pivotal,
Samsung, Schlumberger, Splunk, Virdata and VMware.

{
\bibliography{main}
\bibliographystyle{alpha}
}

\appendix
\COND{}{
\section{Block coordinate descent for random features}
The algorithm we run for block coordinate descent on random feature problems
is described, for completeness, in Algorithm~\ref{alg:rf_bcd}.
\begin{algorithm}[htb]
   \caption{Random features block coordinate descent}
   \label{alg:rf_bcd}
\begin{algorithmic}
   \STATE {\bfseries Input:} data $X \in \X^{n}$, $Y \in \{\pm 1\}^{n \times k}$.
   \STATE ~~~~~~~~~~~ number of epochs $n_e$,
   \STATE ~~~~~~~~~~~ number of random features $p \geq 1$,
   \STATE ~~~~~~~~~~~ block size $b \in \{1, ..., p\}$.
   \STATE ~~~~~~~~~~~ regularizers $\lambda > 0$.
   \STATE {\bfseries Assume:} $p/b$ is an integer.
   \STATE $\pi \gets $ random permutation of $\{1, ..., p\}$.
   \STATE $\mathcal{I}_1, ..., \mathcal{I}_{\frac{p}{b}} \gets$ partition $\pi$ into $\frac{p}{b}$ pieces.
   \STATE $w \gets 0_{p \times k}$.
   \STATE $R \gets 0_{n \times k}$
   \FOR{$\ell=1$ {\bfseries to} $n_e$}
       \STATE $\pi \gets$ random permutation of $\{1, ..., \frac{p}{b} \}$.
       \FOR{$i=1$ {\bfseries to} $\frac{p}{b}$}
           \STATE $I \gets \mathcal{I}_{\pi_i}$.
           \STATE $Z_b \gets \textsc{RandomFeaturesBlock}(X, I)$.
           \STATE $R \gets R - Z_b w(I, [k])$.
           \STATE Solve $(Z_b^\T Z_b + n\lambda I_b) w_b = Z_b^\T(Y - R)$.
           \STATE $R \gets R + Z_b w_b$.
           \STATE $w(I, [k]) \gets w_b$.
       \ENDFOR
   \ENDFOR
\end{algorithmic}
\end{algorithm}

}

\COND{}{
  \section{Scalability of RBF kernel generation}
  \label{sec:appendix:scaling}

  \begin{figure*}[t!]
    \centering
    \subfigure[]{\label{fig:timit-rbf-scale}\includegraphics[width=0.45\textwidth]{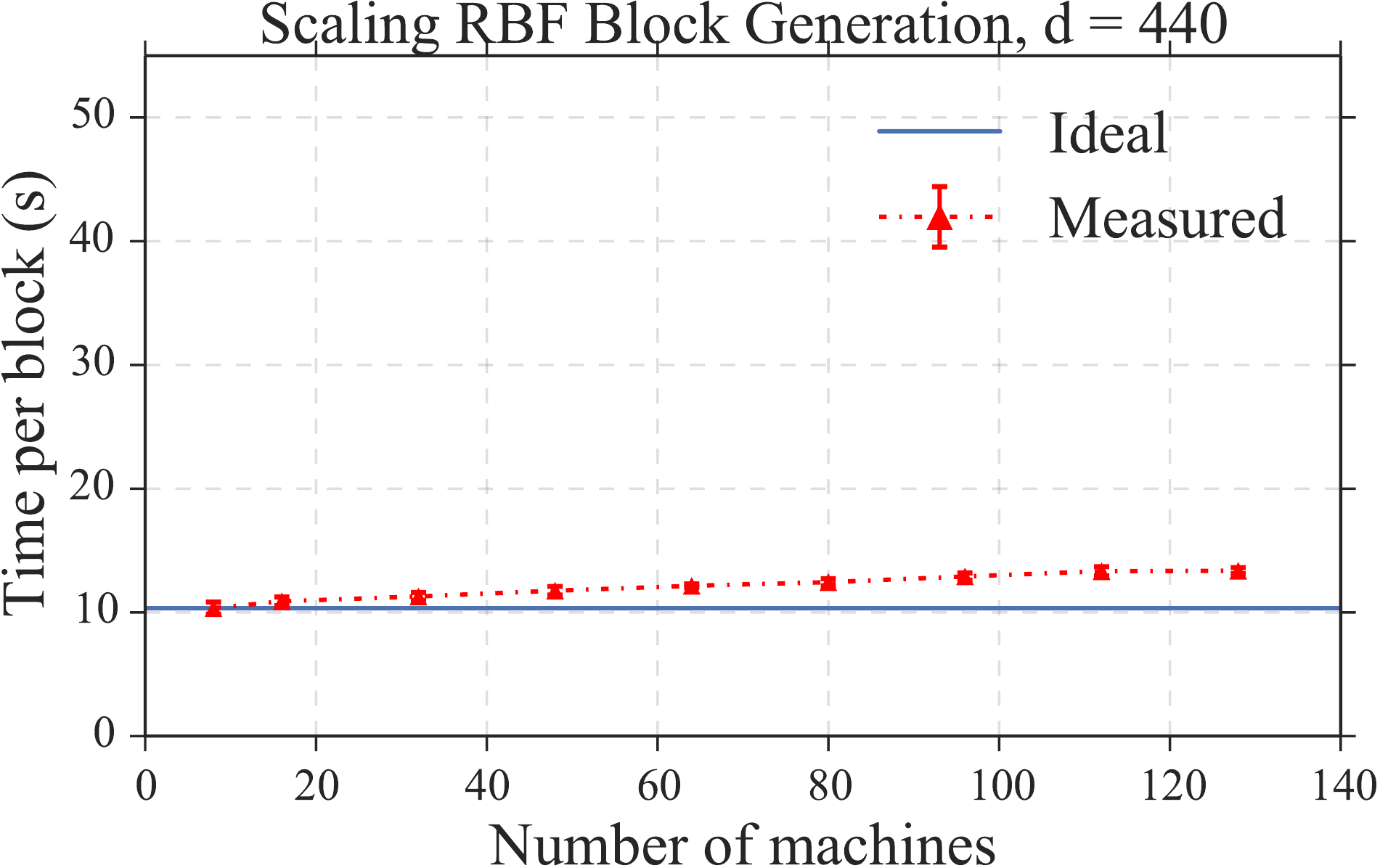}}
    \subfigure[]{\label{fig:cifar-rbf-scale}\includegraphics[width=0.45\textwidth]{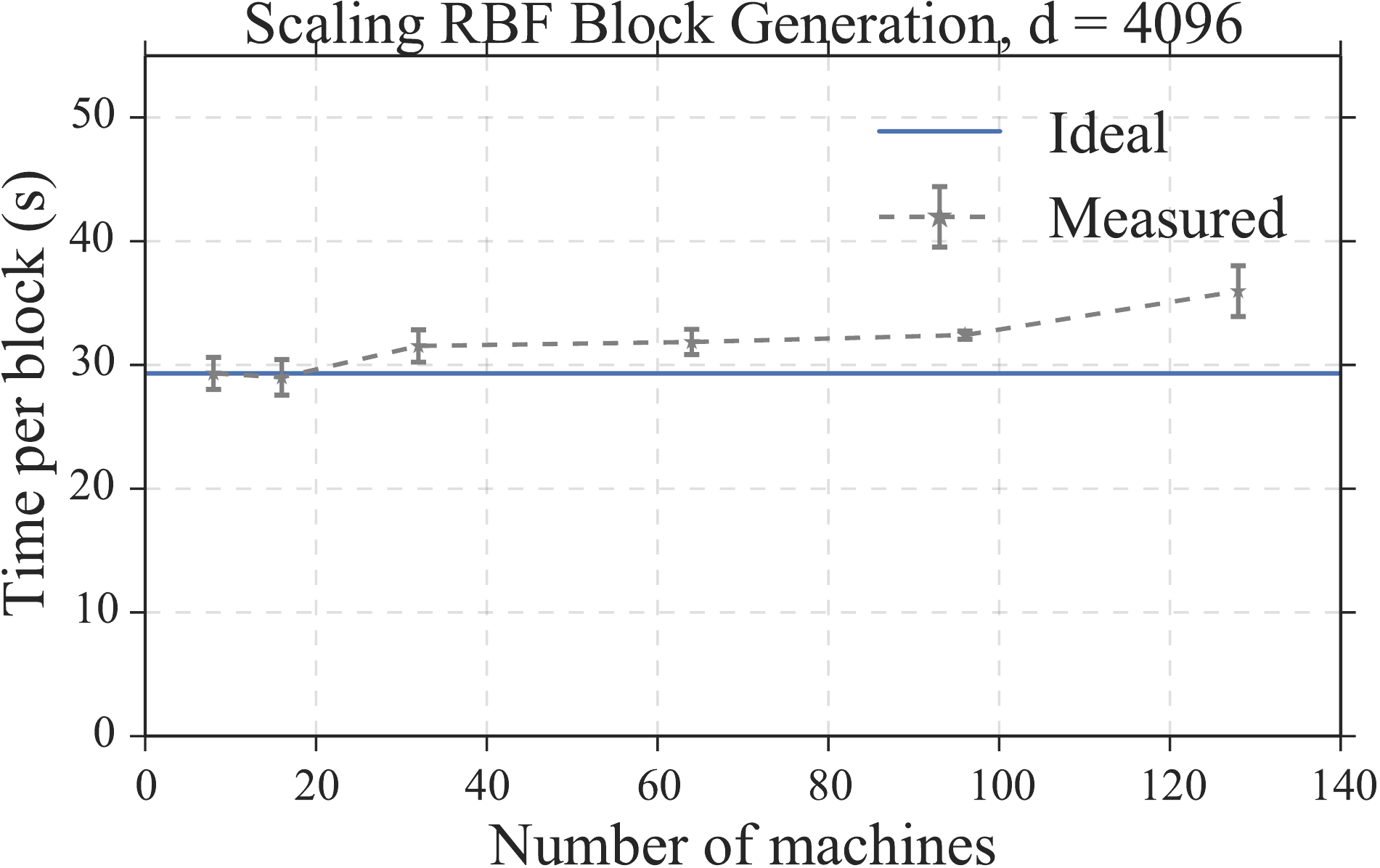}}
    \caption{Time taken to compute one block of the RBF kernel as we scale the number of examples and
  the number of machines used.}
    \label{fig:scaling-rbf}
  \end{figure*}

  To see how the cost of kernel generation changes as dataset size grows, we perform a weak scaling experiment
  where we increase the number of examples and the number of machines used while keeping the number of
  examples per machine constant ($16,384$). We run this experiment for $d=440$ and $d=4096$ which are the
  number of features in TIMIT and CIFAR-10 respectively. Figure~\ref{fig:scaling-rbf} contains results
  from this experiment. In the weak scaling scenario, ideal scaling implies that the time to
  generate a block of the kernel matrix remains constant as we increase both the data and the number of
  machines. However, computing a block of the RBF kernel
  involves broadcasting a $b \times d$ matrix to all the machines in the cluster.
  This causes a slight decrease in performance as we go from $8$ to $128$ machines. As broadcast
  routines scale with $O(log(M))$, we believe that our kernel block generation methods will continue to
  scale well for larger datasets.
}

\section{Proof of Theorem~\ref{thm:bcdrate}}
\label{sec:appendix:a}

Recall that $f : \R^d \rightarrow \R^d$ is a strongly convex and smooth
quadratic function with Hessian $\nabla^2 f \in \R^{d \times d}$. Recall we also assume that
$m I_d \preceq \nabla^2 f \preceq L I_d$.

\paragraph{Notation.}
Let $[d] := \{1, ..., d\}$, $b \in [d]$ be a block size, and let $I \in \Omega_b
:= \{ x \in 2^{[d]} : \abs{x} = b \}$ denote an index set.
Recall that $P_I : \R^{d} \rightarrow \R^{d}$ is the projection operator that zeros
out all the coordinates of the input vector not in $I$, i.e.  $(P_I w)_i = w_i
\ind_{i \in I}$, where $w_i$ denotes the $i$-th coordinate of a vector $w$.  It
is easy to see that $P_I$ in matrix form is $P_I = \diag(\ind_{1 \in I}, ...,
\ind_{d \in I}) \in \R^{d \times d}$.

\paragraph{Block Lipschitz constants.}
We now define a restricted notion of Lipschitz continuity which works on blocks.
For an index set $I$, define
\begin{align*}
  L_I := \sup_{w \in \R^d : \norm{w} = 1} \ip{P_I w}{ \nabla^2 f P_I w} = \lambda_{\max} (P_I \nabla^2 f P_I), \qquad L_{\max,b} := \max_{I \in \Omega_b} L_I \:.
\end{align*}

\paragraph{Update rule.}
Recall that block coordinate descent works by fixing some $w^0 \in \R^d$
and iterating the mapping
\begin{align}
  w^{k+1} \gets \argmin_{w \in \R^d} f(P_{I_k} w + P_{I_k^c} w^{k}) \:, \label{eq:update}
\end{align}
where $I_0, I_1, ... \in \Omega_b$ are chosen by some (random) strategy.  A
common choice is to choose $I_k$ uniformly at random from $\Omega_b$, and to
make this choice independent of the history up to time $k$. This is the sampling strategy
we will study.
We now have enough notation to state and prove our basic inequality for coordinate descent.
This is not new, but we record it for completeness, and because it is simple.
\begin{proposition}
\label{prop:basicinequality}
For every $k \geq 0$, we have that the $k+1$-th iterate
satisfies the inequality
\begin{align}
  f(w^{k+1}) \leq f(w^k) - \frac{1}{2 L_{I_k}} \norm{P_{I_k} \nabla f(w^k)}^2 \:. \label{eq:basicinequality}
\end{align}
\end{proposition}
\begin{proof}
Put $z^k := w^k - \alpha_k P_{I_k} \nabla f(w^k)$.
The update equation in \eqref{eq:update} gives us, trivially, for any $\alpha_k \in \R$,
\begin{align*}
  f(w^{k+1}) \leq f(z^k) \:.
\end{align*}
Now, by Taylor's theorem, for some $t \in (0, 1)$, setting $\alpha_k := 1/L_{I_k}$,
\begin{align*}
  f(z_k) &= f(w^k) - \alpha_k \ip{P_{I_k} \nabla f(w^k)}{ \nabla f(w^k)} + \frac{\alpha_k^2}{2} \ip{P_{I_k} \nabla f(w^k)}{ \nabla^2 f(t w^k + (1-t) (z^k-w^k)) P_{I_k} \nabla f(w^k) } \nonumber \\
                                          &\stackrel{(a)}{\leq} f(w^k) - \alpha_k \norm{P_{I_k} \nabla f(w^k)}^2 + \frac{L_{I_k} \alpha_k^2}{2} \norm{P_{I_k} \nabla f(w^k)}^2 \nonumber \\
                                          &= f(w^k) - \frac{1}{2 L_{I_k}} \norm{P_{I_k} \nabla f(w^k)}^2 \nonumber \:.
\end{align*}
where (a) uses the fact that Euclidean projection is idempotent and also the
definition of $L_{I_k}$.
\end{proof}
We now prove a structural result. The main idea is as follows.  Suppose we have
some subset $\mathcal{G} \subset \Omega_b$ where $\max_{I \in \mathcal{G}} L_I$
is much smaller compared to $L_{\max,b}$.  If this subset is a significant
portion of $\Omega_b$, then we expect to be able to improve the basic rate. The
following result lays the groundwork for us to be able to make this kind of
claim.
\begin{lemma}
\label{lemma:rateset}
Let $\mathcal{G} \subset \Omega_b$ be such that
$\frac{\abs{\mathcal{G}^c}}{\abs{\Omega_b}} = \alpha \frac{b}{d}$ for $\alpha
\in [0, 1]$.  Let each $I_k$ be independent and drawn uniformly from
$\Omega_b$. Then, after $\tau$ iterations, the iterate $w^\tau$ satisfies
\begin{align*}
  \E f(w^\tau) - f_* \leq \left(1 - \frac{b}{d} \left((1-\alpha) \frac{m}{\max_{I \in \mathcal{G}} L_I} + \alpha \frac{m}{L_{\max,t}} \right)   \right)^\tau(f(w^0) - f_*) \:.
\end{align*}
\end{lemma}
\begin{proof}
The basic proof structure is based on Theorem 1 of \cite{wright15}.
The idea here is to compute the conditional expectation of
$\frac{1}{L_{I_k}} \norm{P_{I_k} \nabla f(w^k)}^2$ w.r.t. $w^k$, taking
advantage of the structure provided by $\mathcal{G}$. Put $t := \max_{I \in \mathcal{G}} L_I$.
Then,
\begin{align}
  \E( \frac{1}{L_{I_k}} \norm{P_{I_k} \nabla f(w^k)}^2 | w^k ) &= \E( \frac{1}{L_{I_k}} \norm{P_{I_k} \nabla f(w^k)}^2 \ind_{I_k \in \mathcal{G}} | w^k ) + \E( \frac{1}{L_{I_k}} \norm{P_{I_k} \nabla f(w^k)}^2 \ind_{I_k \not\in \mathcal{G}} | w^k ) \nonumber \\
                                                               &\geq \frac{1}{t} \E( \norm{P_{I_k} \nabla f(w^k)}^2 \ind_{I_k \in \mathcal{G}} | w^k ) + \frac{1}{L_{\max,b}} \E( \norm{P_{I_k} \nabla f(w^k)}^2 \ind_{I_k \not\in \mathcal{G}} | w^k ) \nonumber \\
                                                               &= \nabla f(w^k)^\T \E( \frac{1}{t} P_{I_k} \ind_{I_k \in \mathcal{G}} + \frac{1}{L_{\max,b}} P_{I_k} \ind_{I_k \not\in \mathcal{G}} ) \nabla f(w^k) \nonumber \\
                                                               &\stackrel{(a)}{=} \nabla f(w^k)^\T \E Q_{I_k} \nabla f(w^k) \nonumber \\
                                                               &\geq \lambda_{\min}(\E Q_{I_k}) \norm{\nabla f(w^k)}^2 \label{eq:bcdrate:one} \:,
\end{align}
where in (a) we define $Q_I$ to be diagonal PSD matrix $Q_I :=  \frac{1}{t}
P_{I} \ind_{I \in \mathcal{G}} + \frac{1}{L_{\max,b}} P_{I} \ind_{I \not\in
\mathcal{G}}$.  Let us look at $(\E Q_{I_k})_{\ell \ell}$.  It is not hard to see that
\begin{align*}
    \abs{\Omega_b} (\E Q_{I_k})_{\ell \ell} = \frac{1}{t} \abs{\{ I \in \mathcal{G} : \ell \in I \}} + \frac{1}{L_{\max,b}} \abs{\{ I \in \mathcal{G}^c : \ell \in I \}} \:.
\end{align*}
We know that $\abs{\{ I \in \mathcal{G} : \ell \in I \}} + \abs{\{ I \in \mathcal{G}^c : \ell \in I \}} = \abs{\Omega_b} \frac{b}{d}$.
Since $t \leq L_{\max, b}$, the quantity above is lower bounded when we make $\abs{\{ I \in \mathcal{G}^c : \ell \in I \}}$ as large as possible.
Therefore, since we assume that $\abs{\mathcal{G}^c} \leq \abs{\Omega_b} \frac{b}{d}$,
\begin{align*}
  \frac{1}{t} \abs{\{ I \in \mathcal{G} : \ell \in I \}} + \frac{1}{L_{\max,b}} \abs{\{ I \in \mathcal{G}^c : \ell \in I \}} \geq \frac{1}{t} \left( \abs{\Omega_b} \frac{b}{d} - \abs{\mathcal{G}^c} \right) + \frac{1}{L_{\max,b}} \abs{\mathcal{G}^c} \:,
\end{align*}
from which we conclude
\begin{align}
    \lambda_{\min}(\E Q_{I_k}) \geq \frac{1}{t} \left(\frac{b}{d} - \frac{\abs{\mathcal{G}^c}}{\abs{\Omega_b}}\right) + \frac{1}{L_{\max,b}}\frac{\abs{\mathcal{G}^c}}{\abs{\Omega_b}} = \frac{b}{d} \left((1-\alpha) \frac{1}{t} + \alpha \frac{1}{L_{\max,t}} \right) \:. \label{eq:bcdrate:two}
\end{align}
Combining \eqref{eq:bcdrate:one} and \eqref{eq:bcdrate:two} with
Proposition~\ref{prop:basicinequality} followed by iterating expectations, we
conclude that
\begin{align*}
  \E f(w^{k+1}) \leq \E f(w^k) - \frac{1}{2} \frac{b}{d} \left((1-\alpha) \frac{1}{t} + \alpha \frac{1}{L_{\max,t}} \right) \E \norm{\nabla f(w^k)}^2 \:.
\end{align*}
The rest of the proof proceeds identically to Theorem 1 of \cite{wright15}, using
$m$-strong convexity to control $\norm{\nabla f(w^k)}^2$ from below.
\end{proof}
The remainder of the proof involves showing the existence of a set
$\mathcal{J} \subset \Omega_b$ that satisfies the hypothesis of
Lemma~\ref{lemma:rateset}. To show this, we need some basic tools
from random matrix theory. The following matrix Chernoff inequality is
Theorem 2.2 from \citesup{tropp11}.
\begin{theorem}
\label{thm:chernoff}
Let $\mathcal{X}$ be a finite set of PSD matrices of dimension $k$, and suppose
$\max_{X \in \mathcal{X}} \lambda_{\max}(X) \leq B$.
Sample $\{X_1, ... X_\ell\}$ uniformly at random from $\mathcal{X}$ without
replacement. Put $\mu_{\max} := \ell \cdot \lambda_{\max}(\E X_1)$.
Then, for any $\delta \geq 0$,
\begin{align*}
  \Pr\left\{ \lambda_{\max} (\sum_{j=1}^{\ell} X_j) \geq (1+\delta) \mu_{\max} \right\} \leq k \cdot \left[ \frac{e^\delta}{(1+\delta)^{1+\delta}} \right]^{\mu_{\max}/B} \:.
\end{align*}
\end{theorem}
The inequality of Theorem~\ref{thm:chernoff} can be weakened to a more useful
form, which we will use directly (see e.g. Section 5.1 of \citesup{tropp15}).
The following bound holds for all $t \geq e$,
\begin{align}
  \Pr\left\{ \lambda_{\max} (\sum_{j=1}^{\ell} X_j) \geq t \mu_{\max} \right\} \leq k \cdot (e/t)^{t \mu_{\max}/B} \:. \label{eq:chernoff_weakened}
\end{align}
We now prove, for arbitrary fixed matrices, a result which controls the
behavior of the top singular value of submatrices of our original matrix.  Let
$A \in \R^{n \times p}$ be fixed.  Define for any $t > 0$,
\begin{align*}
  \mathcal{J}_t(A) := \left\{ I \in \Omega_b : \lambda_{\mathrm{max}}(P_I^\T A^\T A P_I) < \frac{tb}{p} \lambda_{\mathrm{max}}(A^\T A) \right\} \:.
\end{align*}
We now establish a result controlling the size of $\mathcal{J}_t(A)$ from below.
We do this via a probabilistic argument, taking advantage of the matrix Chernoff inequality.
\begin{lemma}
\label{lemma:control}
Fix an $A \in \R^{n \times p}$ and $b \in \{1, ..., p\}$ and $\delta \in (0, 1)$.
Suppose that $I$ is drawn uniformly at random from $\Omega_b$.
We have that
\beqs
\Pr\left\{ \lambda_{\max}(P_I A^\T A P_I) \geq e^2 \frac{b}{p} \lambda_{\max}(A^\T A) + \norm{\diag(A^\T A)}_\infty \log\left(\frac{n}{\delta}\right) \right\} \leq \delta \:.
\eeqs
\end{lemma}
\begin{proof}
This argument closely follows Section 5.2.1 of \citesup{tropp15}.
First, we observe that we can write $A P_I$ as $A P_I = \sum_{i=1}^{p} \ind_{\{i \in I\}} A_{:i}
e_i^\T$, where $A_{:i} \in \R^n$ denotes the $i$-th column of $A$. Also, since
$\lambda_{\max}( P_I^\T A^\T A P_I ) = \lambda_{\max}(A P_I P_I^\T
A^\T)$, we focus our efforts on the latter.
Now,
\begin{align*}
  AP_I P_I^\T A^\T = (\sum_{i=1}^{p} \ind_{\{i \in I\}} A_{:i} e_i^\T)(\sum_{i=1}^{p} \ind_{\{i \in I\}} A_{:i} e_i^\T)^\T = \sum_{i=1}^{p} \ind_{\{i \in I\}} A_{:i} A_{:i}^\T \:.
\end{align*}
Let $\mathcal{X} := \{ A_{:i}A_{:i}^\T : i \in \{1, ..., p\} \}$. The
calculation above means we can equivalently view the random variable $A P_I P_I^\T
A^\T$ as the sum $\sum_{i=1}^{b} X_i$ where $X_1, ..., X_b$ are sampled from
$\mathcal{X}$ without replacement.  Put $\mu_{\mathrm{max}} := b \cdot
\lambda_{\mathrm{max}}(\E X_1) = \frac{b}{p} \lambda_{\mathrm{max}} (AA^\T) = \frac{b}{p} \lambda_{\mathrm{max}} (A^\T A)$. Observe that
$\max_{1 \leq i \leq p} \lambda_{\max}( A_{:i} A_{:i}^\T ) = \max_{1 \leq i \leq p} \norm{A_{:i}}^2 := B$.
This puts us in a position to apply Theorem~\ref{thm:chernoff}, using the form
given by \eqref{eq:chernoff_weakened}, from which we conclude for all $t \geq
e$,
\begin{align}
  \Pr\left\{ \lambda_{\max}(P_I^\T A^\T A P_I) \geq \frac{tb}{p} \lambda_{\max}(A^\T A) \right\} \leq n (e/t)^{ \frac{tb}{Bp} \cdot \lambda_{\max}(A^\T A) } \: \label{eq:chernoff_intermediate},
\end{align}
To conclude, set
\beq
  t = \frac{p}{b} \frac{B}{\lambda_{\max}(A^\T A)} \log\frac{n}{\delta} + e^2 \:,\label{eq:setting_of_t}
\eeq
and plug into \eqref{eq:chernoff_intermediate}. The result follows by noting
that $\max_{1 \leq i \leq p} \norm{A_{:i}}^2 = \max_{1 \leq i \leq p} e_i^\T
A^\T A e_i = \max_{1 \leq i \leq p} (A^\T A)_{ii}$.
\end{proof}
We are now in a position to prove Theorem~\ref{thm:bcdrate}.
\begin{proof}
(Of Theorem~\ref{thm:bcdrate}). Let $\nabla^2 f = Q^\T Q$ be a factorization of
$\nabla^2 f$ which exists since $\nabla^2 f$ is PSD. Note that we must have $Q
\in \R^{d \times d}$ because $\nabla^2 f$ is full rank. Recall that
$\lambda_{\max}(Q^\T Q) \leq L$.  First, we note that
\begin{align*}
  \Pr\left\{ \lambda_{\max}(P_I^\T A^\T A P_I) \geq \frac{tb}{p} \lambda_{\max}(A^\T A) \right\} =
  \Pr(I \in \mathcal{J}_t^c(A)) = \E \ind_{\{I \in \mathcal{J}_t^c(A)\}} = \frac{ \abs{\mathcal{J}_t^c(A)} }{ \abs{\Omega_b} } \:.
\end{align*}
Setting $t$ as in \eqref{eq:setting_of_t} and invoking
Lemma~\ref{lemma:control}, we have that every $I \in \mathcal{J}_t(Q)$
satisfies
\begin{align*}
  \lambda_{\max}(P_I Q^\T Q P_I) < \norm{\diag(\nabla^2 f)}_\infty \log(2d^2/b) + e^2 \frac{b}{d} L \:.
\end{align*}
The result follows immediately by an application of Lemma~\ref{lemma:rateset}
\end{proof}

\section{Proofs for Section~\ref{sec:rates:kernelopt}}
\label{sec:appendix:tabrates}

\subsection{Derivation of rates in Table~\ref{tab:rates}}
\label{sec:appendix:tabrates:parta}

\paragraph{Full kernel.}
As noted in Section~\ref{sec:algorithms},
we actually run Gauss-Seidel on $(K+n\lambda I_n) \alpha = Y$,
which can be seen as coordinate descent on the program
\begin{align*}
    \min_{\alpha \in \R^{n \times k}} \frac{1}{2} \ip{\alpha}{K\alpha} + \frac{n\lambda}{2} \norm{\alpha}_F^2 - \ip{Y}{\alpha} \:.
\end{align*}
Note that the objective is a strongly convex function with Hessian given as
$D^2 f(\alpha) [H, H] = \ip{H}{(K+n\lambda I_n)H}$.
Theorem~\ref{thm:bcdrate} tells
us that setting $b \gtrsim n\log{n} \frac{\norm{\diag(K+n\lambda I_n)}_\infty}{\norm{K + n\lambda I_n}}$,
$\widetilde{O}(1/\lambda)$ iterations are sufficient.  Plugging values
in, we get for $b \gtrsim \log^2{n}$ under (a) and
$b \gtrsim n^{1/(2\beta+1)} \log{n}$ under (b),
the number of iterations is bounded
under (a) by $\widetilde{O}(n)$ and
under (b) by $\widetilde{O}( n^{\frac{2\beta}{2\beta+1}} )$.

\paragraph{\nystrom{} approximation.}
We derive a rate for coordinate descent on \eqref{eq:kernelnystrom}. We use the
regularized variant, which ensures that \eqref{eq:kernelnystrom} is strongly
convex. Let $p \leq n$ denote the number of \nystrom{} features,
let $I$ denote the index set of features, and
let $S \in \R^{n \times p}$ be the column selector matrix such that
$K_I = KS$. The Hessian of \eqref{eq:kernelnystrom}
is given by
\beqs
	n D^2 f [H, H] = \ip{H}{(S^\T K(K+n\lambda I_n) S + n\lambda\gamma I_p)H} \:.
\eeqs
Theorem~\ref{thm:bcdrate} tells us that
we want to set $b \gtrsim p\log{p} \frac{\norm{\diag(S^\T K(K+n\lambda I_n) S + n\lambda\gamma I_p)}_\infty}{\norm{S^\T K(K+n\lambda I_n) S + n\lambda\gamma I_p }}$.
Applying a matrix Chernoff argument
(Lemma~\ref{lemma:control})
to control
$\norm{ S^\T K(K+n\lambda I_n) S }$ from above, we have w.h.p. that
the number of iterations is $\widetilde{O}( p/\lambda \gamma  )$.
Under (a) this is $\widetilde{O}(np/\gamma)$ and
under (b) this is $\widetilde{O}(p n^{\frac{2\beta}{2\beta+1}}/\gamma)$.

To control $b$, we apply a matrix Bernstein argument
(Lemma~\ref{lemma:bernstein_control})
to control $\norm{ S^\T
K(K+n\lambda I_n) S }$ from below w.h.p. This argument shows that when $\lambda
\leq O(1)$ and $p \gtrsim \log{n}$, $\norm{ S^\T K(K+n\lambda I_n) S } \gtrsim
np$, from which we conclude that $b \gtrsim (1+\gamma)\log{n}$.

\paragraph{Random features approximation.}
We now derive a rate for coordinate descent on \eqref{eq:randomfeatures}.
The Hessian of \eqref{eq:randomfeatures} is given by
$nD^2 f(\alpha)[H, H] = \ip{H}{( Z^\T Z + n\lambda I_p) H}$.
Thus by Theorem~\ref{thm:bcdrate}, setting
$b \gtrsim p\log{p}\frac{\norm{\diag(Z^\T Z + n\lambda
I_p)}_\infty}{\norm{Z^\T Z + n\lambda I_p}}$ and applying a matrix Bernstein
argument to control $\norm{Z^\T Z + n\lambda I_p}$ from both directions
(Lemma~\ref{lemma:randomfeatures_control}),
then as long as $p
\gtrsim \log{n}$, we have w.h.p. that this is at most
$\widetilde{O}(1/\lambda)$, which is the same rate as the full kernel.
Furthermore, the block size is $b \gtrsim \log{n}$.

\subsection{Supporting lemmas for Section~\ref{sec:appendix:tabrates:parta}}

For a fixed symmetric $Q$ and random $I$, we want to control $\lambda_{\max}(
P_I Q P_I )$ from below. The matrix Chernoff arguments do not allow us to do
this, so we rely on matrix Bernstein.  The following result is Theorem 2 from
\cite{alaoui15}.
\begin{theorem}
\label{thm:bernstein}
Fix a matrix $\Psi \in \R^{n \times m}$, $p \in \{1, ..., m\}$ and $\beta \in
(0, 1]$.  Let $\psi_i \in \R^{n}$ denote the $i$-th column of $\Psi$. Choose
$i_k$, $k=1,...,p$ from $\{1, ..., m\}$ such that $\Pr(i_k = j) = p_i \geq \beta
\norm{\psi_i}^2/\norm{\Psi}_F^2$. Put $\widetilde{S} \in \R^{n \times p}$ such that
$\widetilde{S}_{ij} = 1/\sqrt{ p \cdot p_{ij}}$ if $i=i_j$ and $0$ otherwise. Then for all $t \geq 0$,
\beqs
    \Pr\left\{ \lambda_{\max}( \Psi\Psi^\T - \Psi \widetilde{S}\widetilde{S}^\T \Psi^\T ) \geq t  \right\} \leq n \exp\left( \frac{-pt^2/2}{\lambda_{\max}(\Psi\Psi^\T)( \norm{\Psi}^2_F/\beta + t/3) } \right) \:.
\eeqs
\end{theorem}
This paves the way for the following lemma.
\begin{lemma}
\label{lemma:bernstein_control}
Fix a matrix $\Psi \in \R^{n \times m}$, $p \in \{1, ..., m\}$.  Let $\psi_i
\in \R^{n}$ denote the $i$-th column of $\Psi$.  Put $B := \max_{1 \leq i \leq m} \norm{\psi_i}^2$.  Choose $I := (i_1, ..., i_p)$
uniformly at random without replacement from $\{1, ..., m\}$.  Let $S \in \R^{n
\times p}$ be the column selector matrix associated with $I$. Then, with
probability at least $1 - \delta$ over the randomness of $I$,
\beqs
    \lambda_{\max}( \Psi SS^\T \Psi^\T ) \geq \frac{p}{m} \lambda_{\max}( \Psi\Psi^\T ) - \frac{4}{3} \frac{\lambda_{\max}(\Psi\Psi^\T)}{m} \log\left(\frac{n}{\delta}\right) - \sqrt{ \frac{8p}{m} \lambda_{\max}( \Psi\Psi^\T )B \log\left(\frac{n}{\delta}\right)} \:.
\eeqs
\end{lemma}
\begin{proof}
Put $p_i = 1/m$ for $i=1,...,m$ and and $\beta = \frac{\norm{\Psi}_F^2}{ m B
}$. By definition, $\beta \leq 1$.  In this case, $\widetilde{S} =
\sqrt{\frac{m}{p}} S$. Plugging these constants into
Theorem~\ref{thm:bernstein}, we get that
\beqs
    \Pr\left\{ \lambda_{\max}( \frac{p}{m} \Psi\Psi^\T - \Psi SS^\T \Psi^\T ) \geq \frac{p}{m} t  \right\} \leq n \exp\left( \frac{-pt^2/2}{\lambda_{\max}(\Psi\Psi^\T)( m B + t/3) } \right) \:.
\eeqs
Setting the RHS equal to $\delta$, we get that $t$ is the roots of the quadratic equation
\beqs
    t^2 - \frac{2}{3} \frac{\lambda_{\max}(\Psi\Psi^\T)}{p} \log\left(\frac{n}{\delta}\right) \cdot t - 2 \lambda_{\max}( \Psi\Psi^\T ) \frac{mB}{p} \log\left(\frac{n}{\delta}\right) = 0 \:.
\eeqs
Since solutions to $t^2 - at - b = 0$ satisfy $t \leq 2(a + \sqrt{b})$ when $a, b \geq 0$, from this we conclude
\beqs
    t \leq \frac{4}{3} \frac{\lambda_{\max}(\Psi\Psi^\T)}{p} \log\left(\frac{n}{\delta}\right) + \sqrt{8 \lambda_{\max}( \Psi\Psi^\T ) \frac{mB}{p} \log\left(\frac{n}{\delta}\right)} \:.
\eeqs
Hence,
\beq
    \Pr\left\{ \lambda_{\max}( \frac{p}{m} \Psi\Psi^\T - \Psi SS^\T \Psi^\T ) \geq \frac{4}{3} \frac{\lambda_{\max}(\Psi\Psi^\T)}{m} \log\left(\frac{n}{\delta}\right) + \sqrt{ \frac{8p}{m} \lambda_{\max}( \Psi\Psi^\T )B \log\left(\frac{n}{\delta}\right)}  \right\} \leq \delta \:. \label{eq:boundone}
\eeq
By the convexity of $\lambda_{\max}(\cdot)$,
\beq
    \frac{p}{m} \lambda_{\max}( \Psi\Psi^\T )  = \lambda_{\max}( \Psi SS^\T \Psi^\T + \frac{p}{m} \Psi\Psi^\T - \Psi SS^\T \Psi^\T ) \leq \lambda_{\max}( \Psi SS^\T \Psi^\T ) + \lambda_{\max}(\frac{p}{m} \Psi\Psi^\T - \Psi SS^\T \Psi^\T ) \:. \label{eq:boundtwo}
\eeq
Combining \eqref{eq:boundone} and \eqref{eq:boundtwo} yields the result.
\end{proof}
We now study random features. To do this, we need the following
general variant of matrix Bernstein. The following is Corollary 6.2.1 of \citesup{tropp15}.
\begin{theorem}
\label{thm:bernstein_tropp}
Let $B \in \R^{d_1 \times d_2}$ be a fixed real matrix. Let $R \in \R^{d_1
\times d_2}$ be a random matrix such that $\E R = K$ and $\norm{R} \leq L$ a.s.
Put
\begin{align*}
  G := \frac{1}{n} \sum_{k=1}^{n} R_k, \qquad m_2(R) := \max\{ \norm{\E RR^\T}, \norm{\E R^\T R} \} \:,
\end{align*}
where each $R_k$ is an independent copy of $R$. Then for all $t \geq 0$,
\begin{align*}
    \Pr \left\{\norm{G - K} \geq t \right\}  \leq (d_1 + d_2) \exp\left( \frac{-nt^2/2}{m_2(R) + 2Lt/3} \right) \:.
\end{align*}
\end{theorem}
This variant allows us to easily establish the following lemma.
\begin{lemma}
\label{lemma:randomfeatures_control}
Fix an $\alpha \in (0, 1)$. Let $ZZ^\T \in \R^{n \times n}$ be from the random features construction.
Put $B := \sup_{x \in \X} \abs{\varphi(x, \omega)}$, and suppose $B < \infty$.
Then with probability at least $1 - \delta$, we have that
as long as $p \geq \frac{2}{\alpha} (\frac{1}{\alpha} + \frac{2}{3})\frac{n B^2}{\norm{K}} \log\left( \frac{2n}{\delta} \right)$,
\begin{align*}
  (1-\alpha) \norm{K} \leq \norm{ZZ^\T} \leq (1+\alpha) \norm{K} \:.
\end{align*}
\end{lemma}
\begin{proof}
We set up parameters so we can invoke Theorem~\ref{thm:bernstein_tropp}. This
follows Section 6.5.5 of \citesup{tropp15}.  Define
$\xi_k = (\varphi(x_1, \omega_k), ..., \varphi(x_n, \omega_k)) \in \R^n$ and
$R_k = \xi_k \xi_k^\T$. This setting means that $\frac{1}{p} \sum_{k=1}^{p} R_k
= ZZ^\T$.  We have $\norm{R_k} = \norm{\xi_k \xi_k^\T} = \norm{\xi_k}^2 \leq
nB^2$.  Furthermore,
\begin{align*}
  \E R_k^2 = \E \norm{\xi_k}^2 \xi_k\xi_k^\T \preceq nB^2 \E \xi_k\xi_k^\T = nB^2 K \Longrightarrow m_2(R) \leq nB^2 \norm{K} \:.
\end{align*}
Hence by Theorem~\ref{thm:bernstein_tropp},
\begin{align*}
    \Pr\left\{ \norm{ ZZ^\T - K} \geq t \right\} \leq 2n \exp\left( \frac{-pt^2/2}{nB^2\norm{K} + 2nB^2 t/3} \right) \:.
\end{align*}
Setting $t = \alpha \norm{K}$, we require that $p \geq \frac{2}{\alpha}
(\frac{1}{\alpha} + \frac{2}{3})\frac{n B^2}{\norm{K}} \log\left(
\frac{2n}{\delta} \right)$ to ensure that $\Pr\left\{ \norm{ ZZ^\T - K} \geq
\alpha\norm{K} \right\} \leq \delta$.  On the complement on this event, we have
that $\norm{ZZ^\T} \leq \norm{K} + \norm{ZZ^\T - K} \leq (1+\alpha)\norm{K}$.
Similarly, $\norm{K} \leq \norm{ZZ^\T} + \norm{K - ZZ^\T} \leq \norm{ZZ^\T} + \alpha \norm{K}$.
The result now follows.
\end{proof}

\end{document}